%% file: aaai2020.tex
\documentclass[letterpaper,english]{article}
\usepackage[latin9]{inputenc}
\usepackage{array}
\usepackage{multirow}
\usepackage{algorithm2e}
\usepackage{amsmath}
\usepackage{amsthm}
\usepackage{amssymb}
\usepackage{graphicx}
\usepackage{xargs}[2008/03/08]
\usepackage{microtype}

\makeatletter

\pdfpageheight\paperheight
\pdfpagewidth\paperwidth

\providecommand{\tabularnewline}{\\}

\theoremstyle{plain}
\newtheorem{thm}{\protect\theoremname}
\theoremstyle{plain}
\newtheorem{lem}[thm]{\protect\lemmaname}

\def\year{2020}\relax
\usepackage{aaai20}  % DO NOT CHANGE THIS
\usepackage{times}  % DO NOT CHANGE THIS
\usepackage{helvet} % DO NOT CHANGE THIS
\usepackage{courier}  % DO NOT CHANGE THIS
\usepackage[hyphens]{url}  % DO NOT CHANGE THIS
\usepackage{graphicx} % DO NOT CHANGE THIS
\usepackage{graphicx}  % DO NOT CHANGE THIS
\usepackage{colortbl} 
\definecolor{hd_cl}{rgb}{0.74,0.88,0.91}
\definecolor{header_color}{rgb}{0.74,0.88,0.91}
\definecolor{ev_cl}{rgb}{0.9,0.9,0.9}
\definecolor{even_color}{rgb}{0.9,0.9,0.9}
\definecolor{subhd_cl}{rgb}{0.85,0.93,0.95}
\definecolor{subheader_color}{rgb}{0.85,0.93,0.95}
\definecolor{childheader_color}{rgb}{1.0,0.93,0.87}

\@ifundefined{showcaptionsetup}{}{%
 \PassOptionsToPackage{caption=false}{subfig}}
\usepackage{subfig}
\makeatother

\usepackage{babel}
\providecommand{\lemmaname}{Lemma}
\providecommand{\theoremname}{Theorem}

\begin{document}
\include{macros}

\title{Bayesian Optimization for Categorical and Category-Specific Continuous
Inputs}

\author{Dang Nguyen, Sunil Gupta, Santu Rana, Alistair Shilton, Svetha Venkatesh\\
Applied Artificial Intelligence Institute (A\textsuperscript{2}I\textsuperscript{2}),
Deakin University, Geelong, Australia\\
\textit{\{d.nguyen, sunil.gupta, santu.rana, alistair.shilton, svetha.venkatesh\}@deakin.edu.au}}
\maketitle
\begin{abstract}
Many real-world functions are defined over both categorical and \textit{category-specific}
continuous variables and thus cannot be optimized by traditional Bayesian
optimization (BO) methods. To optimize such functions, we propose
a new method that formulates the problem as a multi-armed bandit problem,
wherein each category corresponds to an arm with its reward distribution
centered around the optimum of the objective function in continuous
variables. Our goal is to identify the best arm and the maximizer
of the corresponding continuous function simultaneously. Our algorithm
uses a Thompson sampling scheme that helps connecting both multi-arm
bandit and BO in a unified framework. We extend our method to \textit{batch}
BO to allow parallel optimization when multiple resources are available.
We theoretically analyze our method for convergence and prove sub-linear
regret bounds. We perform a variety of experiments: optimization of
several benchmark functions, hyper-parameter tuning of a neural network,
and automatic selection of the best machine learning model along with
its optimal hyper-parameters (a.k.a \textit{automated machine learning}).
Comparisons with other methods demonstrate the effectiveness of our
proposed method.
\end{abstract}

\section{Introduction\label{sec:Introduction}}

\input{introduction.tex}

\section{Related Background\label{sec:Related-Work}}

\input{relatedwork.tex}

\section{The Proposed Method\label{sec:Framework}}

\input{framework_part1.tex}

\input{framework_part2.tex}

\section{Discussion\label{sec:Discussion}}

\input{discussion.tex}

\section{Experiments\label{sec:Experiments}}

\input{experiment.tex}

\section{Conclusion\label{sec:Conclusion}}

\input{conclusion.tex}

\section*{Acknowledgment}

This research was partially funded by the Australian Government through
the Australian Research Council (ARC). Prof Venkatesh is the recipient
of an ARC Australian Laureate Fellowship (FL170100006).

\bibliographystyle{aaai}
\bibliography{aaai2020}
\newpage{}

\appendix

\part*{Supplementary Material}

\section{Theoretical Analysis}

Under Assumption 1 (refer to the main paper), when using Thompson
sampling (TS) for BO of a function $f_{c}(x)$ where $x\in\mathcal{X}_{c}$,
the \emph{Bayesian regret} after $T_{c}$ iterations is given as \cite{russo2014learning,kandasamy2018parallelised}

\begin{align}
\text{BayesRegret}(T_{c}) & =\mathbb{E}\sum_{t_{c}=1}^{T_{c}}[f_{c}(x_{c}^{*})-f_{c}(x_{t_{c}}^{c})]\nonumber \\
 & \leq\mathcal{O}\left(\sqrt{T_{c}\text{log}T_{c}\gamma_{T_{c}}}\right)\label{eq:BayesRegret_definition}
\end{align}
where $\gamma_{T_{c}}$ is the maximum information gain about $f_{c}(x)$
due to any $T_{c}$ function observations and the expectation is with
respect to the distribution over all possible functions $f_{c}$ in
our hypothesis space and any randomness in the algorithm, particularly
the random sampling of TS. 

Let us define \emph{Bayesian simple regret} as 
\[
\text{BayesSimpleRegret}(T_{c})\triangleq\mathbb{E}[f_{c}(x_{c}^{*})-\max_{t_{c}\leq T_{c}}f_{c}(x_{t_{c}}^{c})]
\]

Since $\max_{t_{c}\leq T_{c}}f_{c}(x_{t_{c}}^{c})\geq\frac{1}{T_{c}}\sum_{t_{c}=1}^{T_{c}}f_{c}(x_{t_{c}}^{c})$,
from Eq. (\ref{eq:BayesRegret_definition}), we have
\begin{align}
\text{BayesSimpleRegret}(T_{c}) & =\mathbb{E}[f_{c}(x_{c}^{*})-\max_{t_{c}\leq T_{c}}f_{c}(x_{t_{c}}^{c})]\nonumber \\
 & \leq\mathcal{O}\left(\sqrt{\frac{\text{log}T_{c}\gamma_{T_{c}}}{T_{c}}}\right)\label{eq:BayesSimpleRegret_Definition}
\end{align}

\subsection{Proof of Lemma 1 (Upper Bound on $R_{T}^{\textrm{MAB}}$)}
\begin{proof}
To have the proof, we use the connection between the TS and the UCB
algorithms as established by \cite{russo2014learning}. We can write
the regret $R_{T}^{\text{MAB}}$ as 
\begin{align}
R_{T}^{\text{MAB}} & =\mathbb{E}\sum_{t=1}^{T}\left(f_{c^{*}}(x^{*})-f_{c_{t}}(x_{c}^{*})\right)\nonumber \\
 & =\mathbb{E}\sum_{t=1}^{T}(f_{c^{*}}^{*}-U_{t}(c^{*}))+\mathbb{E}\sum_{t=1}^{T}(U_{t}(c^{*})-f_{c_{t}}^{*})\nonumber \\
 & =\underset{\text{Term-1}}{\underbrace{\mathbb{E}\sum_{t=1}^{T}(f_{c^{*}}^{*}-U_{t}(c^{*}))}}+\underset{\text{Term-2}}{\underbrace{\mathbb{E}\sum_{t=1}^{T}(U_{t}(c_{t})-f_{c_{t}}^{*})}}\label{eq:two-terms}
\end{align}
In the above expression, we have defined $f_{c^{*}}^{*}\triangleq f_{c^{*}}(x^{*})$
and $f_{c_{t}}^{*}\triangleq f_{c_{t}}(x_{c}^{*})$. Also, we have
introduced an upper confidence bound function $U_{t}(c)$. The last
equality follows as conditioned on $D_{t}$, the optimal arm $c^{*}$
and the arm $c_{t}$ selected by TS are identically distributed and
$U_{t}$ is deterministic \cite{russo2014learning}. We provide a
detailed explanation. Note that here we are analyzing Bayesian regret
(instead of the \textquotedblleft usual\textquotedblright{} regret),
which is defined over a set $\mathcal{F}$ of problem instances along
with a probability distribution over $\mathcal{F}$. Our function
$f(c,x)$ is just one of these problem instances and can be considered
a random function from $\mathcal{F}$. Note that $c^{*}$ being the
optimal arm for $f(c,x)$ is then also a random variable. Since we
do not know the function $f(c,x)$, but only have $t$ observations
from it, therefore, given $D_{t}$ $c^{*}$ is distributed as $p_{t}(c)=P(c=c^{*}|D_{t})$.
Since Thompson sampling precisely uses this posterior distribution
to propose $c_{t}$ at iteration $t$, both $c_{t}$ and $c^{*}$are
identically distributed conditioned on $D_{t}$. This argument is
fundamental in the Bayesian regret proof of Thompson sampling and
was first used by Russo and Van Roy \cite{russo2014learning}. Several
papers since then have used this argument e.g. see \cite{Bubeck_Liu_2013prior},
where step-1 in the proof of Theorem 1 exactly uses this argument.

The uncertainty in $f_{c}^{*}$ for $c$-th arm depends on the uncertainty
of the Gaussian process (GP) posterior, which reduces with increasing
evaluations of $f_{c}(x)$. To write an appropriate upper confidence
bound for $f_{c}^{*}$, we can use the TS regret analysis of BO. In
Eq. (\ref{eq:BayesSimpleRegret_Definition}), plugging $\gamma_{t_{c}}\sim\mathcal{O}(t_{c}^{\alpha})$
(Assumption 2), we have $\mathbb{E}[f_{c}^{*}-\max_{\{t'|c_{t}=c\}}f(x_{t'})]\leq\mathcal{O}(\sqrt{\frac{\log t_{c}}{t_{c}^{1-\alpha}}})$.
Given this inequality, we can define an upper confidence bound $U_{t}(c)$
on $f_{c}^{*}$ as follows $f_{c}^{*}\leq U_{t}(c)=\mathbb{E}[\max_{\{t'|c_{t}=c\}}f(x_{t'})]+a\sqrt{\frac{\log t_{c}}{t_{c}^{1-\alpha}}})$,
where $t_{c}$ is the number of times the arm $c$ has been selected
in the first $t$ iterations and $a$ is an appropriate positive constant.
We also define a lower confidence bound $L_{t}\left(c\right)=\mathbb{E}[\max_{\{t'|c_{t}=c\}}f(x_{t'})]$,
which holds trivially \emph{i.e.} $L_{t}(c)=\mathbb{E}[\max_{\{t'|c_{t}=c\}}f(x_{t'})]\leq f_{c}^{*}$. 

The Term-1 of Eq. (\ref{eq:two-terms}) is always negative as $f_{c^{*}}^{*}\leq U_{t}(c^{*})$
by the definition of $U_{t}(c^{*})$ and therefore can be ignored
when considering an upper bound on $R_{T}^{\text{MAB}}$. We will
next derive an upper bound on the Term-2 of Eq. (\ref{eq:two-terms}).
For this, consider 
\begin{align*}
\mathbb{E}\sum_{t=1}^{T}(U_{t}(c_{t})-f_{c_{t}}^{*}) & \leq\mathbb{E}\sum_{t=1}^{T}(U_{t}(c_{t})-L_{t}(c_{t}))\\
 & =a\sum_{c=1}^{C}\sum_{t_{c}=1}^{T_{c}}\sqrt{\frac{\log t_{c}}{t_{c}^{1-\alpha}}})\\
 & \leq a\sqrt{\log T}\sum_{c=1}^{C}\sum_{t_{c}=1}^{T_{c}}\frac{1}{\sqrt{t_{c}^{1-\alpha}}}
\end{align*}
 We write $\sum_{t_{c}=1}^{T_{c}}\frac{1}{\sqrt{t_{c}^{1-\alpha}}}=\sum_{t_{c}=1}^{T_{c}}\frac{t_{c}^{\alpha/2}}{\sqrt{t_{c}}}\leq T_{c}^{\alpha/2}\sum_{t_{c}=1}^{T_{c}}\frac{1}{\sqrt{t_{c}}}$
then using the identity $\sum_{t_{c}=1}^{T_{c}}\frac{1}{\sqrt{t_{c}}}\leq\int_{0}^{T_{c}}\frac{1}{\sqrt{r}}dr=2\sqrt{T_{c}}$
we obtain
\[
\mathbb{E}\sum_{t=1}^{T}(U_{t}(c_{t})-f_{c_{t}}^{*})\leq2a\sqrt{\log T}\sum_{c=1}^{C}\sqrt{T_{c}^{\alpha+1}}
\]
By using the Cauchy-Schwarz inequality, we can bound $\sum_{c=1}^{C}\sqrt{T_{c}^{\alpha+1}}\leq\sqrt{C\sum_{c=1}^{C}T_{c}^{\alpha+1}}$.
Next we set $q_{c}=T_{c}^{\alpha+1}$ and $\eta=1/(\alpha+1)$ in
the identity $\left(\sum_{c=1}^{C}q_{c}\right)^{\eta}\leq\sum_{c=1}^{C}q_{c}^{\eta}$,
where $0<\eta\leq1$, to get $\sum_{c=1}^{C}T_{c}^{\alpha+1}\leq\left(\sum_{c=1}^{C}T_{c}\right)^{\alpha+1}$.
Therefore, we can write $\mathbb{E}\sum_{t=1}^{T}(U_{t}(c_{t})-f_{c_{t}}^{*})\leq2a\sqrt{\log T}\sqrt{C\left(\sum_{c=1}^{C}T_{c}\right)^{\alpha+1}}=2a\sqrt{CT^{\alpha+1}\log T}$
and we get $R_{T}^{\text{MAB}}\leq2a\sqrt{CT^{\alpha+1}\log T}$.
\end{proof}

\subsection{Proof of Lemma 2 (Upper Bound on $R_{T}^{\textrm{BO}}$)}
\begin{proof}
By definition, we have $R_{T}^{\text{BO}}=\mathbb{E}\sum_{t=1}^{T}[f_{c_{t}}(x_{c_{t}}^{*})-f_{c_{t}}(x_{t})]=\mathbb{E}\sum_{c=1}^{C}\sum_{t_{c}=1}^{T_{c}}[f_{c}(x_{c}^{*})-f_{c}(x_{t_{c}})]$.
From Eq. (\ref{eq:BayesRegret_definition}), we have $\mathbb{E}\sum_{t_{c}=1}^{T_{c}}[f_{c}(x_{c}^{*})-f_{c}(x_{t_{c}})]\leq\mathcal{O}\left(\sqrt{T_{c}\text{log}T_{c}\gamma_{T_{c}}}\right)\leq b\sqrt{T_{c}^{\alpha+1}\log T_{c}}$
under Assumption 2. Therefore, we have $R_{T}^{\text{BO}}\leq b\mathbb{E}\sum_{c=1}^{C}\sqrt{T_{c}^{\alpha+1}\log T_{c}}\leq b\sqrt{\log T}\mathbb{E}\sum_{c=1}^{C}\sqrt{T_{c}^{\alpha+1}}$.
As in the proof of Lemma 1, we can bound as $\sum_{c=1}^{C}\sqrt{T_{c}^{\alpha+1}}\leq\sqrt{C\sum_{c=1}^{C}T_{c}^{\alpha+1}}\leq\sqrt{CT^{\alpha+1}}$
and therefore have $R_{T}^{\text{BO}}\leq b\sqrt{CT^{\alpha+1}\log T}$.
\end{proof}

\subsection{Proof of Theorem 4 (Batch Setting)}
\begin{proof}
The Bayesian regret for the batch scheme can be written similar to
the sequential case as
\[
\text{BayesRegret}(T)=R_{T}^{\text{MAB}}+R_{T}^{\text{BO}},
\]
where $R_{T}^{\text{MAB}}\triangleq\mathbb{E}\sum_{t=1}^{T}[f_{c^{*}}(x^{*})-f_{c_{t}}(x_{c_{t}}^{*})]$
and $R_{T}^{\text{BO}}\triangleq\mathbb{E}\sum_{t=1}^{T}[f_{c_{t}}(x_{c_{t}}^{*})-f_{c_{t}}(x_{t})]$.

Similar to the sequential case, for batch setting too, we have 
\begin{align*}
R_{T}^{\text{MAB}} & =\underset{\text{Term-1}}{\underbrace{\mathbb{E}\sum_{t=1}^{T}(f_{c^{*}}^{*}-U_{t}(c^{*}))}}+\underset{\text{Term-2}}{\underbrace{\mathbb{E}\sum_{t=1}^{T}(U_{t}(c_{t})-f_{c_{t}}^{*})}}
\end{align*}

Similar to Eq. (\ref{eq:BayesRegret_definition}) and (\ref{eq:BayesSimpleRegret_Definition})
of the sequential setting, we have, for batch setting, under Assumptions
1 and 2 of the main paper, TS \textit{\emph{Bayesian regret}} and
\textit{\emph{Bayesian simple regret}} (for BO of $f_{c}(x)$) bounded
by $\mathcal{O}(\psi_{B}\sqrt{T_{c}^{\alpha+1}\log T_{c}})$ and $\mathcal{O}(\psi_{B}\sqrt{\frac{\log T_{c}}{T_{c}^{1-\alpha}}})$
respectively \cite{kandasamy2018parallelised}. Using the Bayesian
simple regret we can write $\mathbb{E}[f_{c}^{*}-\max_{\{t'|c_{t}=c\}}f(x_{t'})]\leq\mathcal{O}(\psi_{B}\sqrt{\frac{\log t_{c}}{t_{c}^{1-\alpha}}})$.
Given this inequality, we can define an upper confidence bound on
$f_{c}^{*}$ as $f_{c}^{*}\leq U_{t}(c)=\mathbb{E}[\max_{\{t'|c_{t}=c\}}f(x_{t'})]+a\psi_{B}\frac{\log t_{c}}{\sqrt{t_{c}^{1-\alpha}}}$,
where $t_{c}$ is the number of times the arm $c$ has been selected
in the first $t$ iterations and $a$ is an appropriate positive constant.
As before we define a lower confidence bound $L_{t}\left(c\right)=\mathbb{E}[\max_{\{t'|c_{t}=c\}}f(x_{t'})]$,
which holds trivially \emph{i.e.} $L_{t}(c)\leq f_{c}^{*}$. 

The remaining proof follows the same set of arguments as in the sequential
case (see the proof of Lemma 1) and we obtain $\mathbb{E}\sum_{t=1}^{T}(U_{t}(c_{t})-f_{c_{t}}^{*})\leq2a\psi_{B}\sqrt{CT^{\alpha+1}\log T}$
and since $\mathbb{E}\sum_{t=1}^{T}(f_{c^{*}}^{*}-U_{t}(c^{*}))$
is always negative, we have, for the batch setting, $R_{T}^{\text{MAB}}\leq2a\psi_{B}\sqrt{CT^{\alpha+1}\log T}$. 

The analysis of finding the upper bound on $R_{T}^{\text{BO}}$ mimics
the proof steps in Lemma 1 and for our batch setting, it is upper
bounded as $R_{T}^{\text{BO}}\leq\mathbb{E}\sum_{c=1}^{C}b\psi_{B}\sqrt{T_{c}^{\alpha+1}\log T_{c}}\leq b\psi_{B}\sqrt{\log T}\mathbb{E}\sum_{c=1}^{C}\sqrt{T_{c}^{\alpha+1}}$,
where $b$ being an appropriate constant. As in the proof of Lemma
1, we can prove $\sum_{c=1}^{C}\sqrt{T_{c}^{\alpha+1}}\leq\sqrt{CT^{\alpha+1}}$
and therefore, $R_{T}^{\text{BO}}\leq b\psi_{B}\sqrt{CT^{\alpha+1}\log T}$.

Finally combining $R_{T}^{\text{MAB}}$ and $R_{T}^{\text{BO}}$,
we have $\text{BayesRegret}(T)$ for our \textbf{batch setting} bounded
by $2a\psi_{B}\sqrt{CT^{\alpha+1}\log T}+b\psi_{B}\sqrt{CT^{\alpha+1}\log T}$,
which can be summarized as a sub-linear term in $T$ as $\mathcal{O}(\psi_{B}\sqrt{CT^{\alpha+1}\log T})$.
\end{proof}

\section{Additional Experimental Results}

\subsection{Experiments with Other Synthetic Functions}

\subsubsection{2d function (1 categorical variable + 1 continuous variable)}

We compare our \textbf{Bandit-BO} with other methods on a 2d function
as follows

\[
f([c,x])=\exp(-(z_{1}-2)^{2})+\exp(\frac{-(z_{1}-6)^{2}}{10})+\frac{1}{z_{2}^{2}+1}+\frac{c}{2},
\]
where $z_{1}=x-0.05\times c$, $z_{2}=x+0.05\times c$, and $x\in[-2,10]$.

Figures \ref{fig:Optimization-results-2d-function-b} and \ref{fig:Optimization-results-2d-function-c}
show the optimization results for the function with different batch
sizes and different numbers of categories respectively. From the results,
we can see that \textbf{Bandit-BO} is the best method followed by
SMAC and TPE.

\begin{figure*}
\begin{centering}
\includegraphics[scale=0.25]{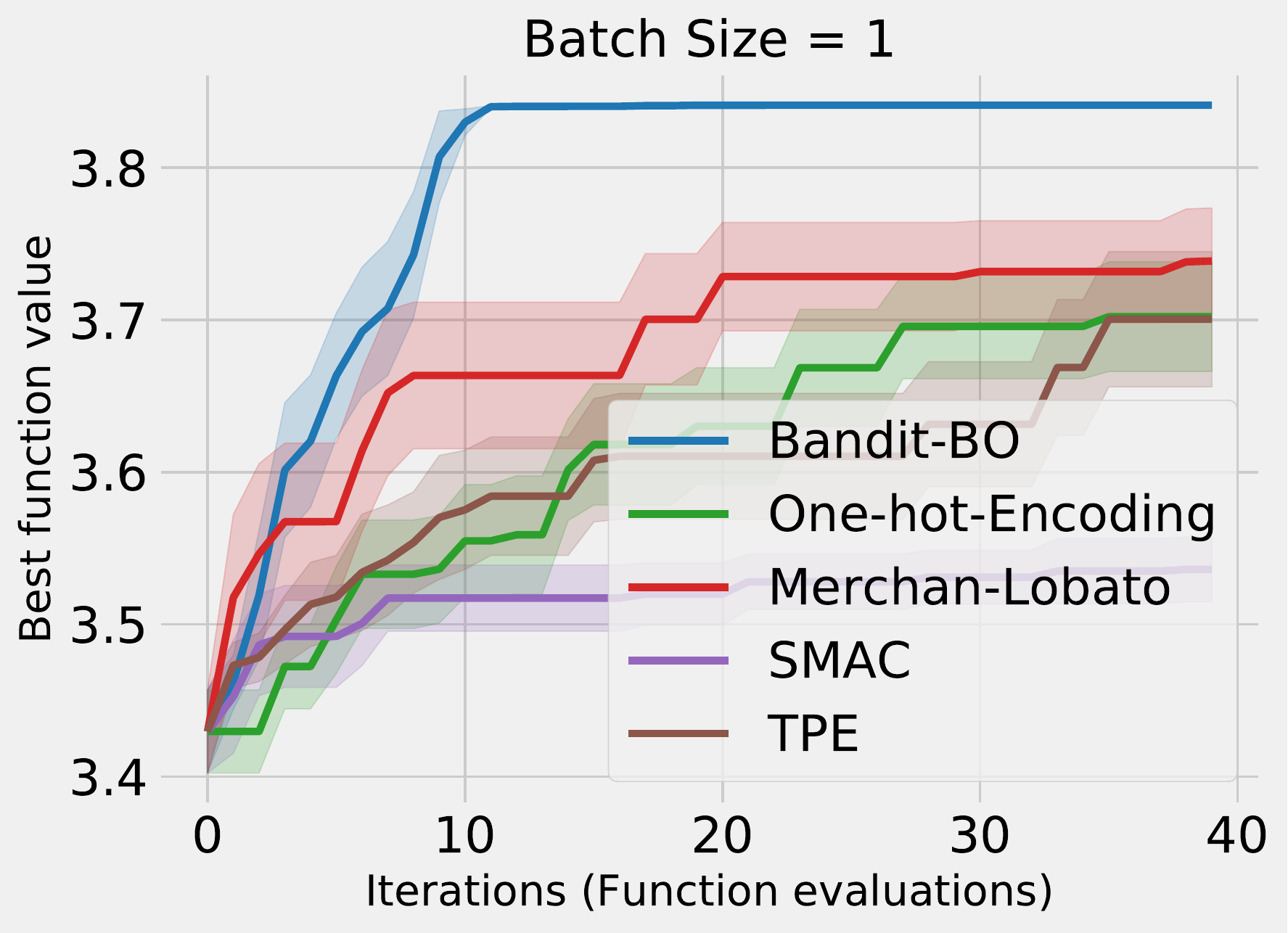}\hspace{0.4cm}\includegraphics[scale=0.25]{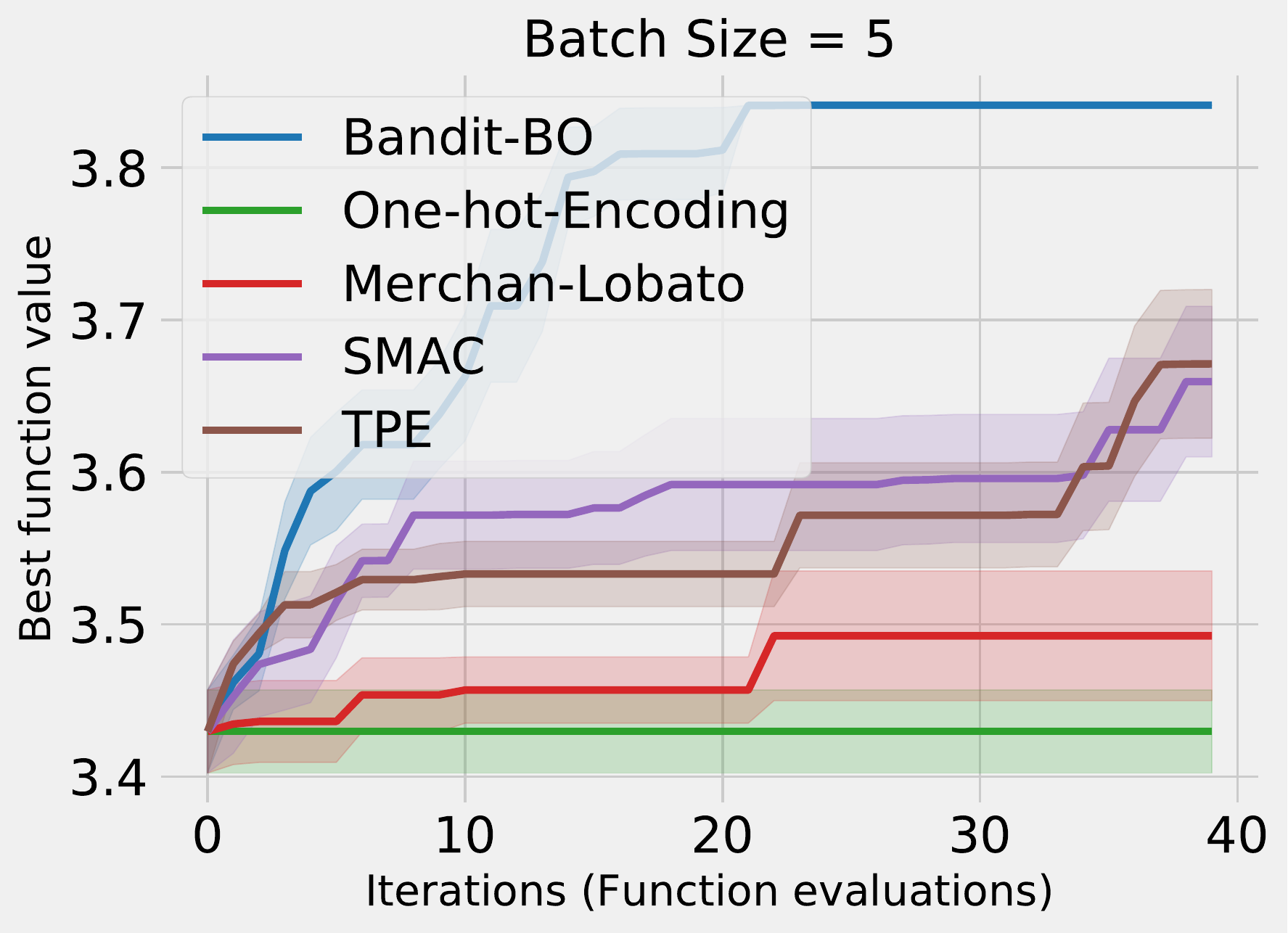}\hspace{0.4cm}\includegraphics[scale=0.25]{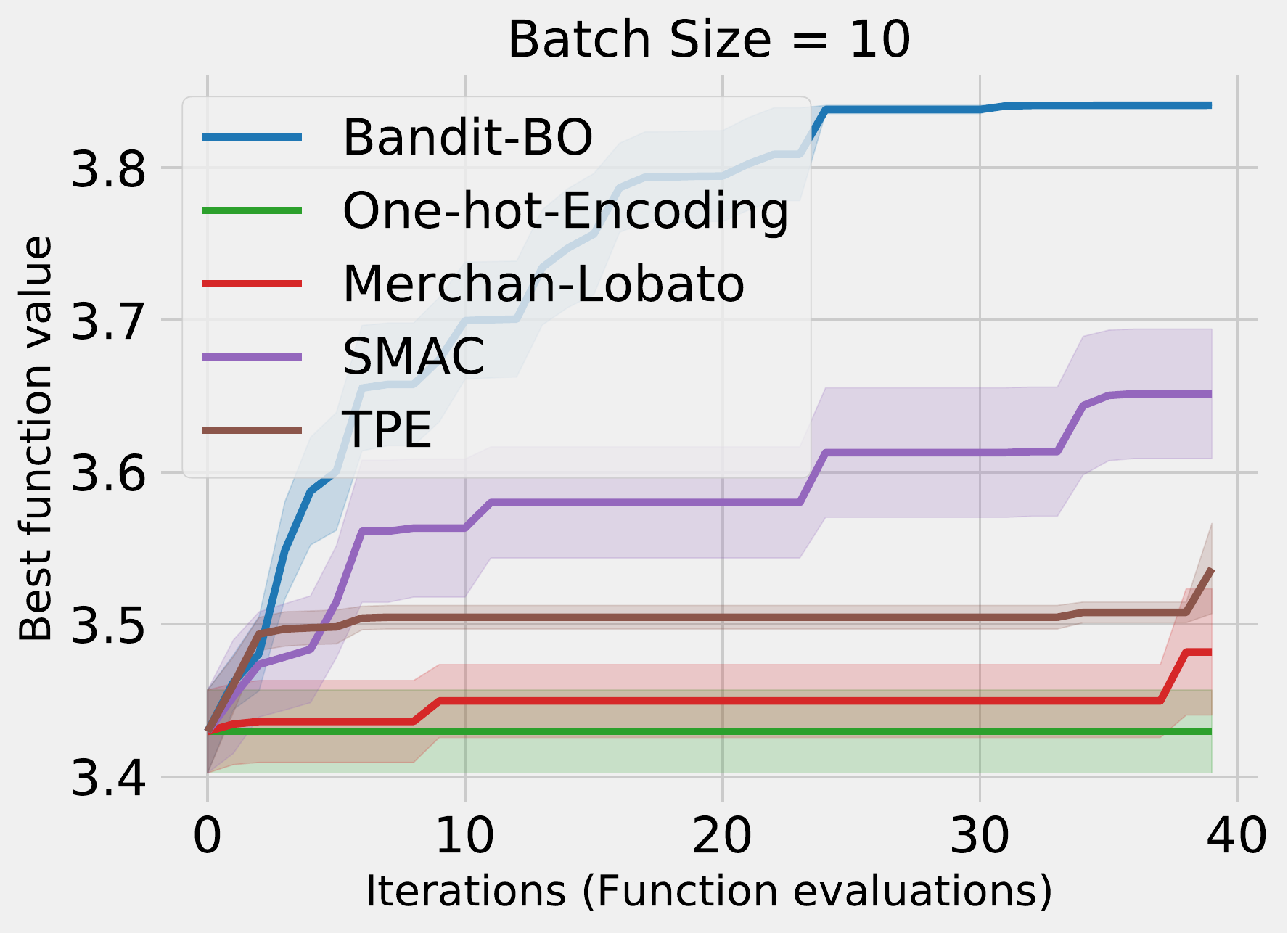}
\par\end{centering}
\caption{\label{fig:Optimization-results-2d-function-b}Optimization results
for the test 2d function for different batch sizes $B\in\{1,5,10\}$.
The number of categories is fixed to 6.}
\end{figure*}

\begin{figure*}
\begin{centering}
\includegraphics[scale=0.25]{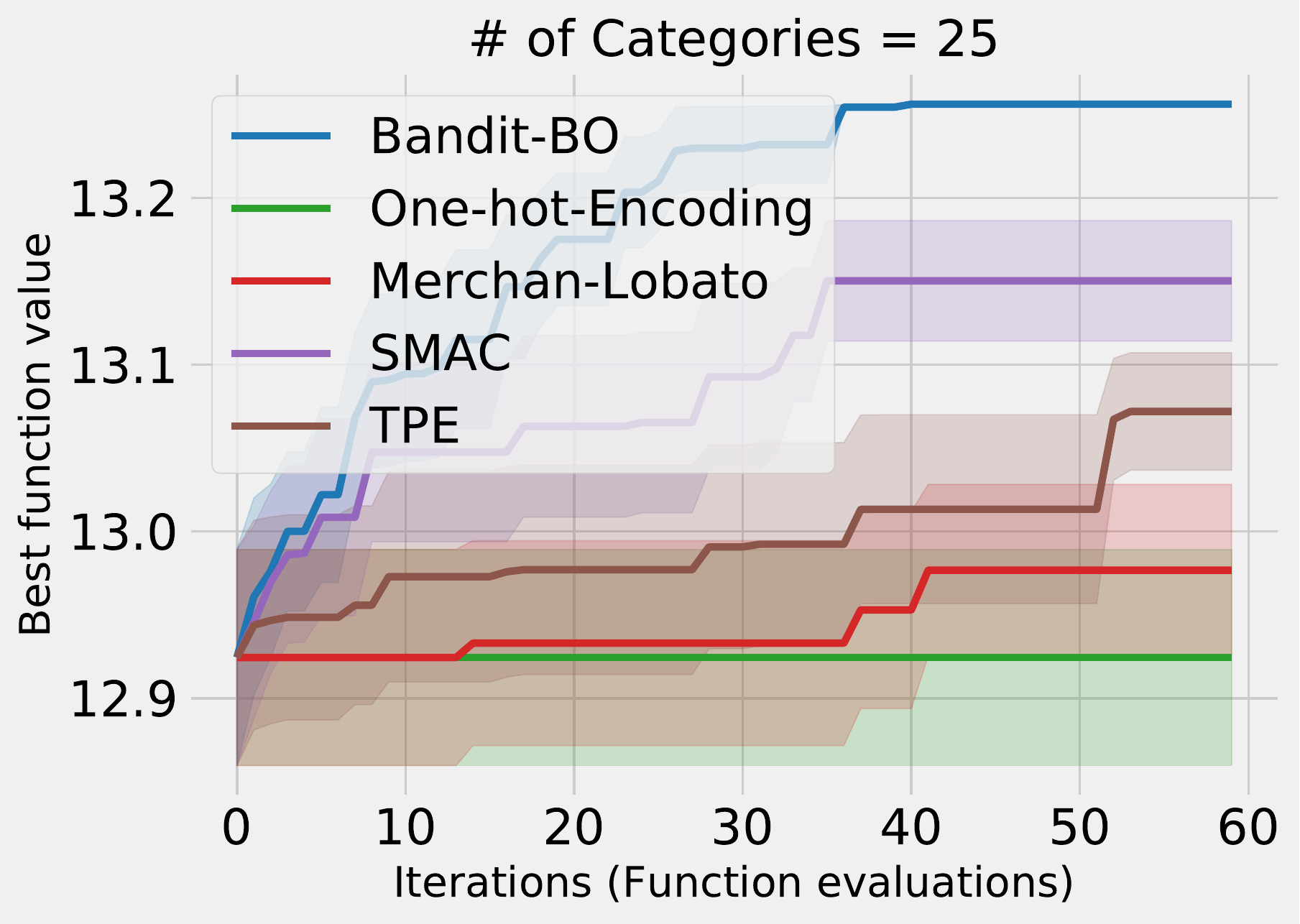}\hspace{0.4cm}\includegraphics[scale=0.25]{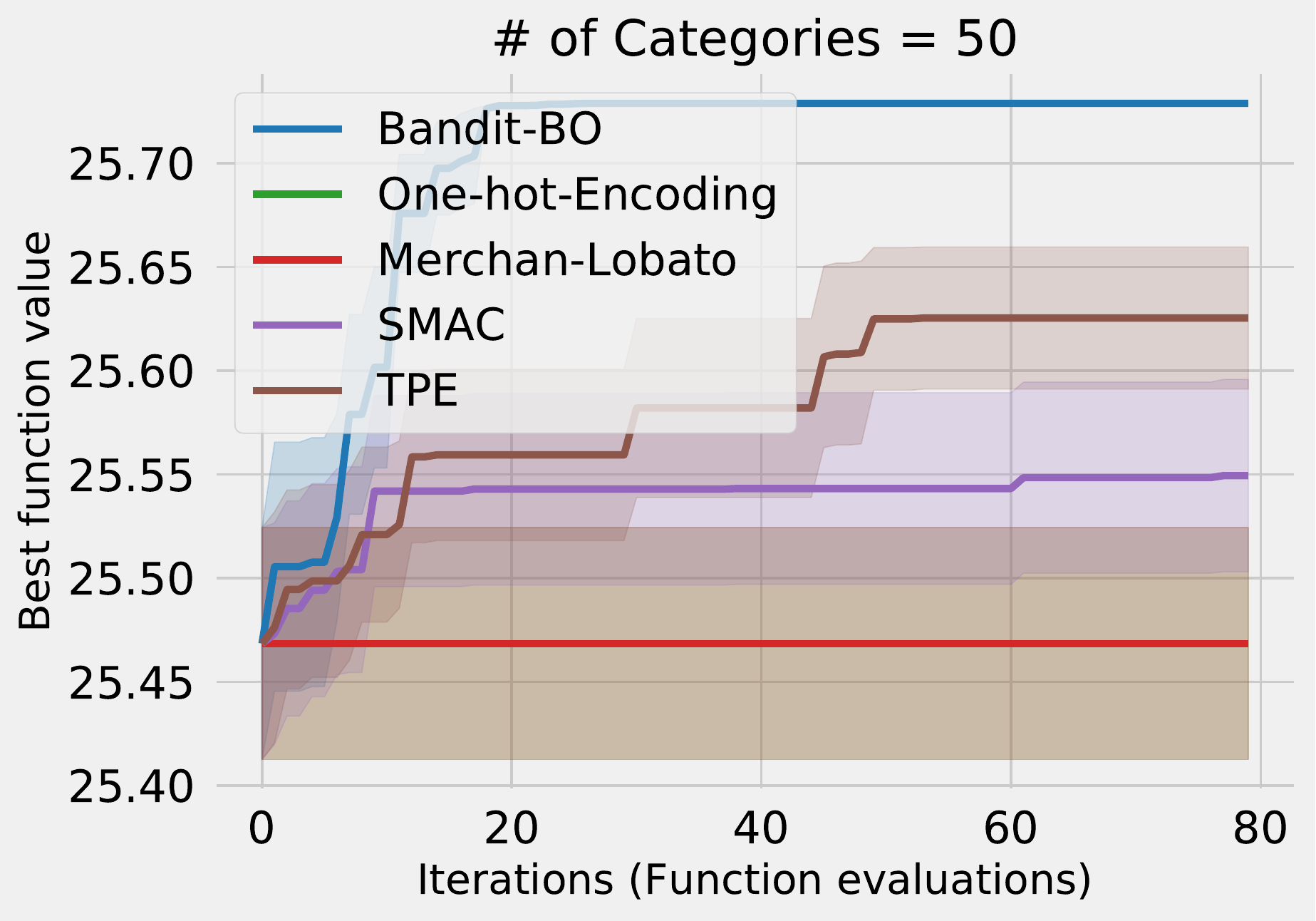}\hspace{0.4cm}\includegraphics[scale=0.25]{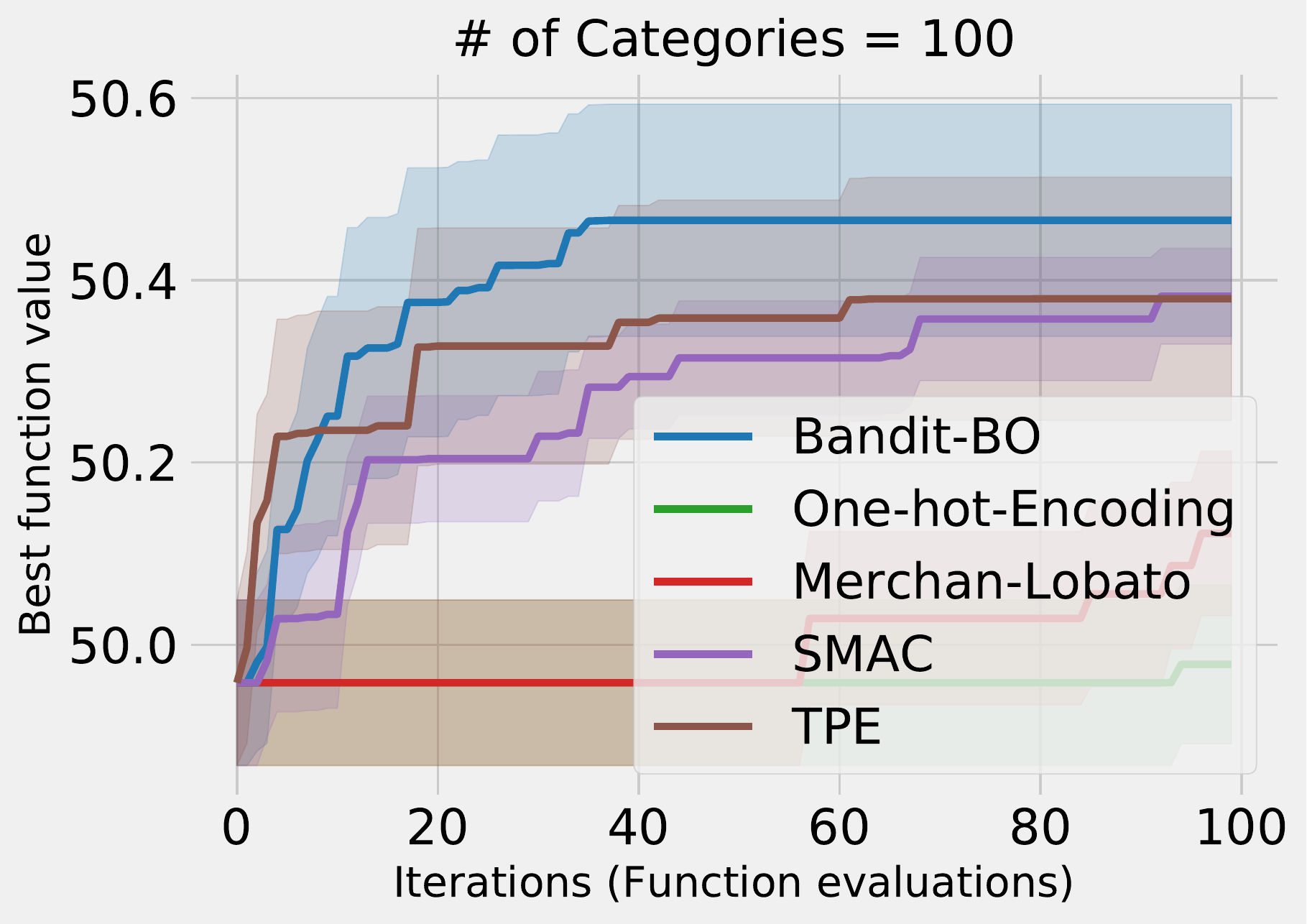}
\par\end{centering}
\caption{\label{fig:Optimization-results-2d-function-c}Optimization results
for the test 2d function for different numbers of categories $C\in\{25,50,100\}$.
The batch size is fixed to 5.}
\end{figure*}

\begin{figure*}
\begin{centering}
\includegraphics[scale=0.25]{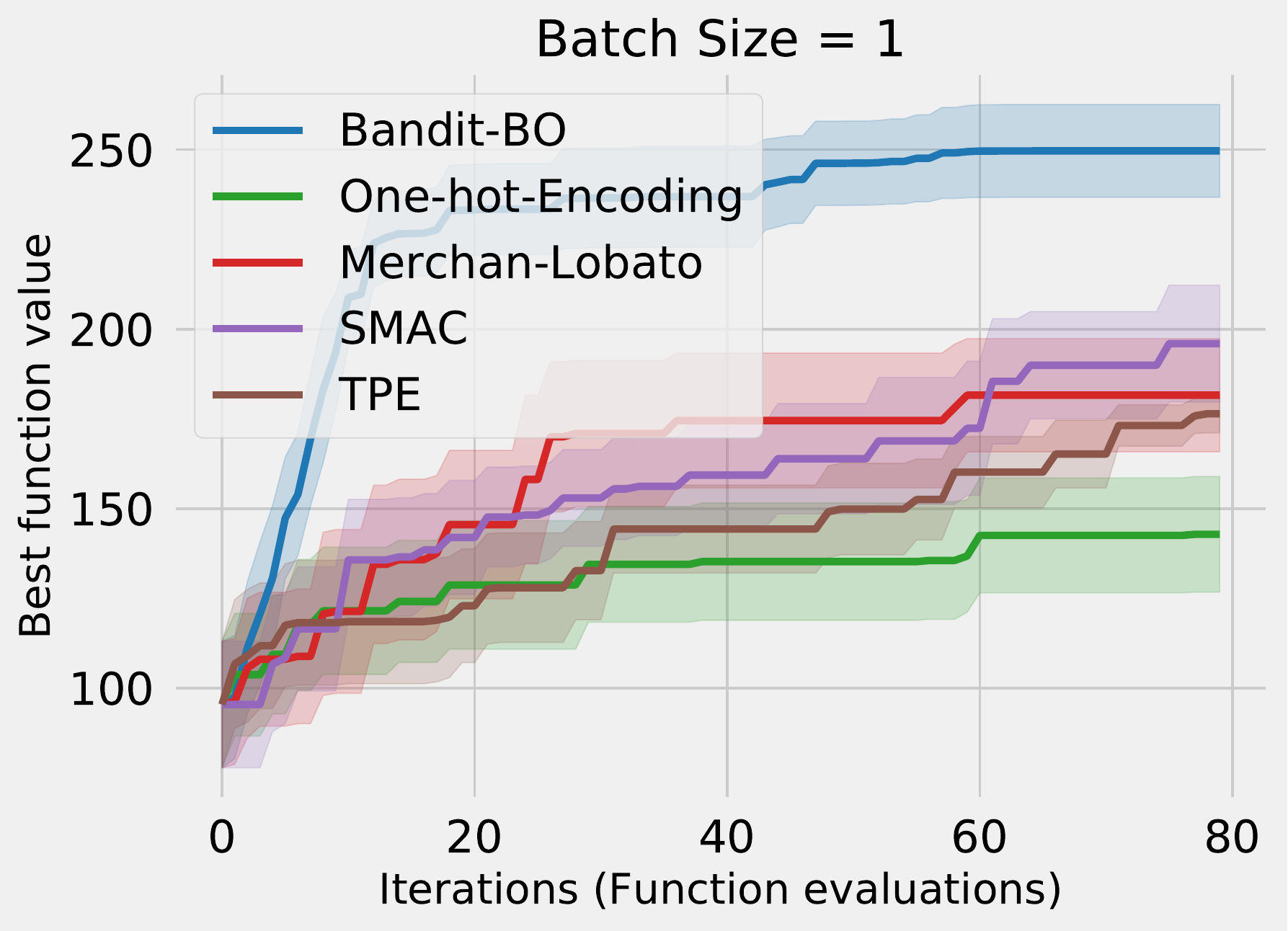}\hspace{0.4cm}\includegraphics[scale=0.25]{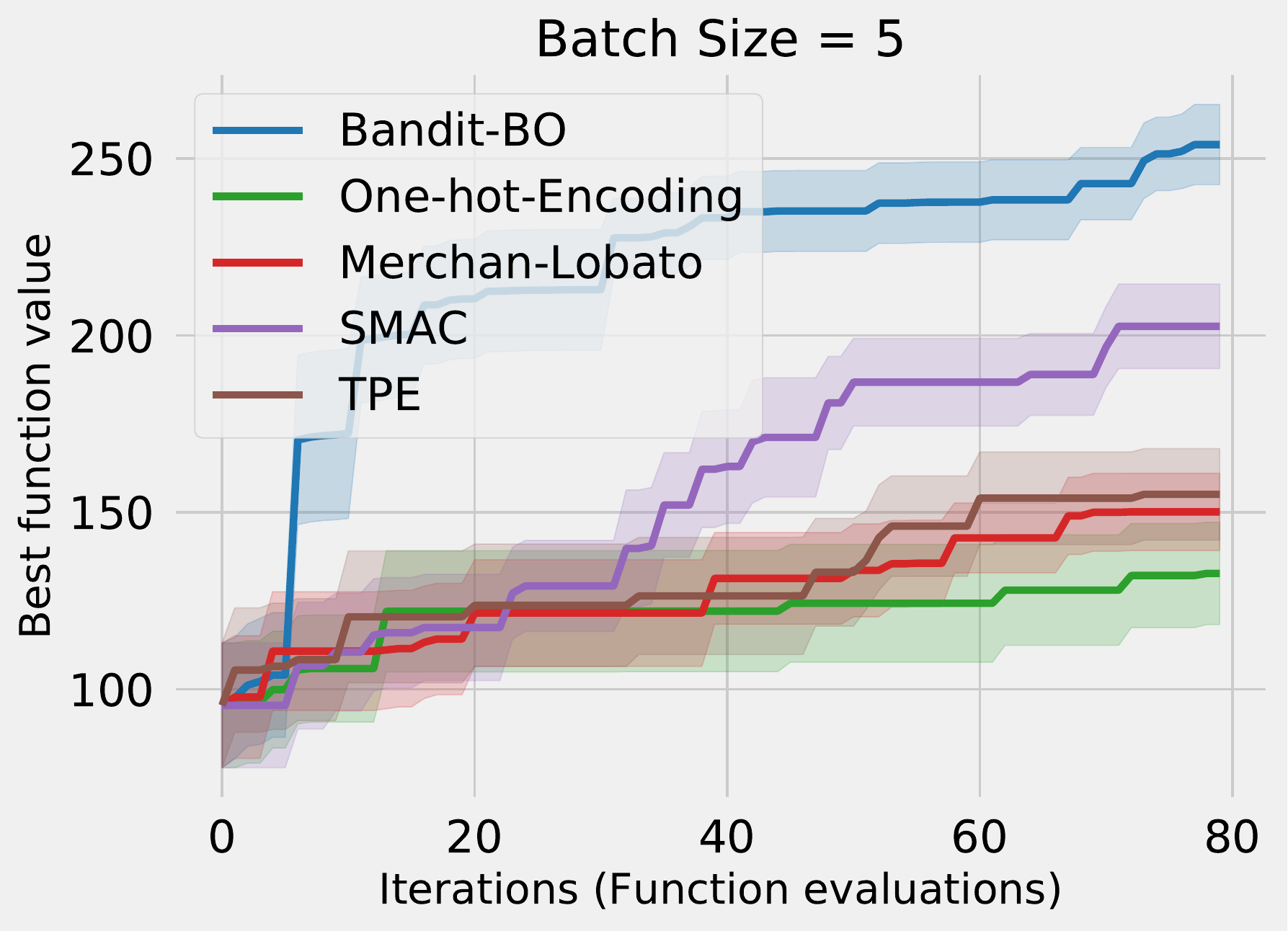}\hspace{0.4cm}\includegraphics[scale=0.25]{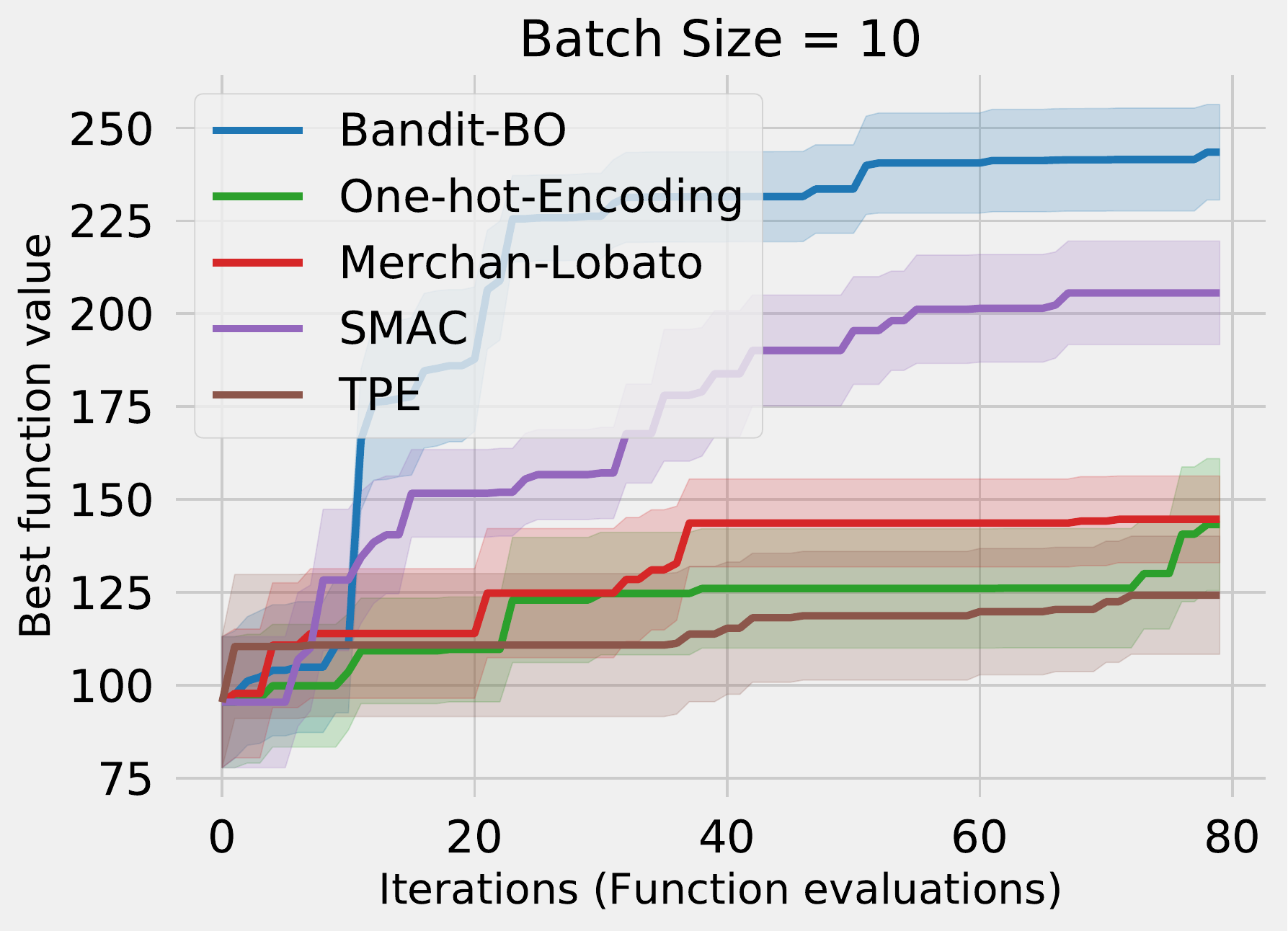}
\par\end{centering}
\caption{\label{fig:Optimization-results-5d-function-b}Optimization results
for the \textit{modified} Alpine-4d function for different batch sizes
$B\in\{1,5,10\}$. The number of categories is fixed to 6.}
\end{figure*}

\begin{figure*}
\begin{centering}
\includegraphics[scale=0.25]{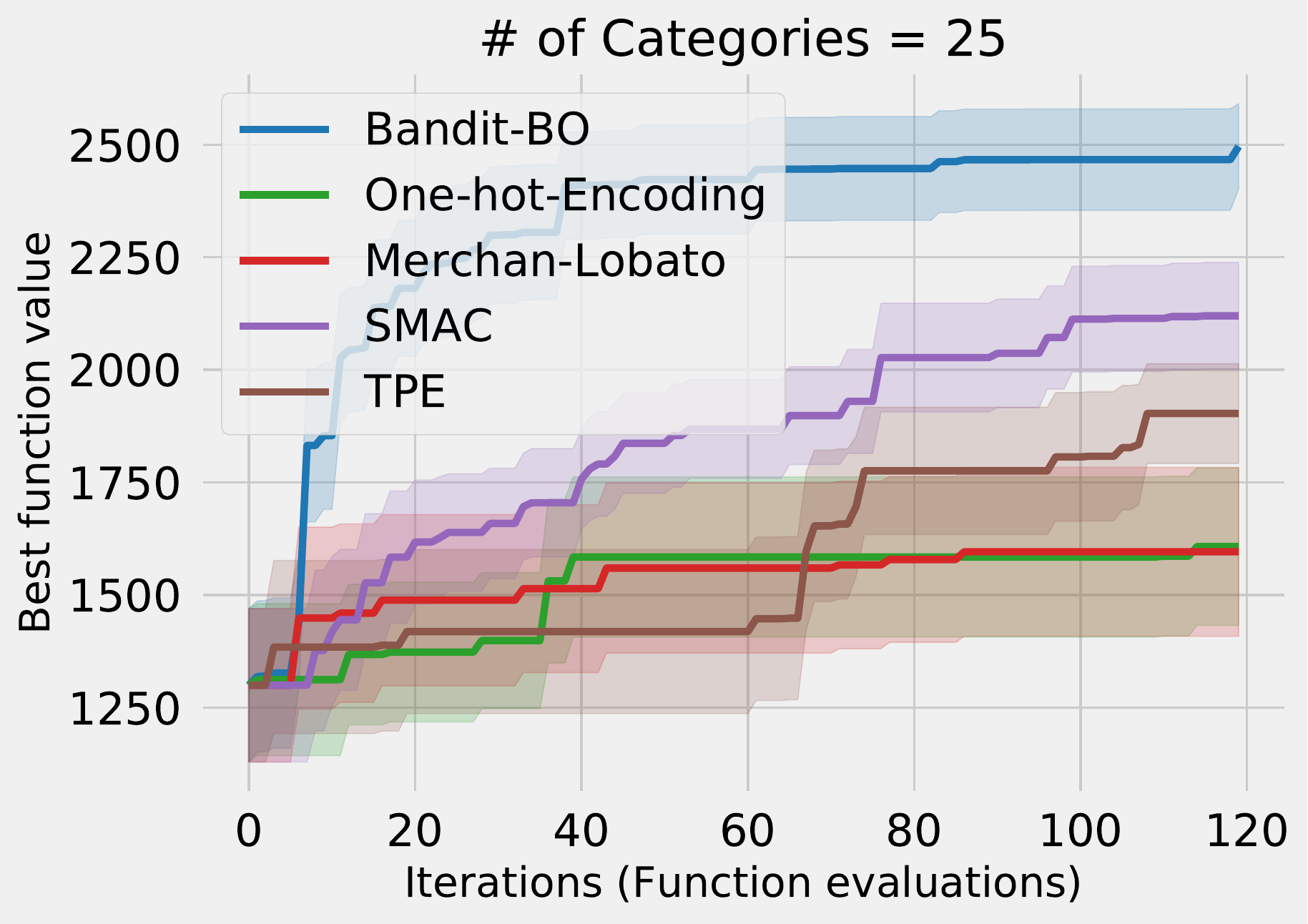}\hspace{0.4cm}\includegraphics[scale=0.25]{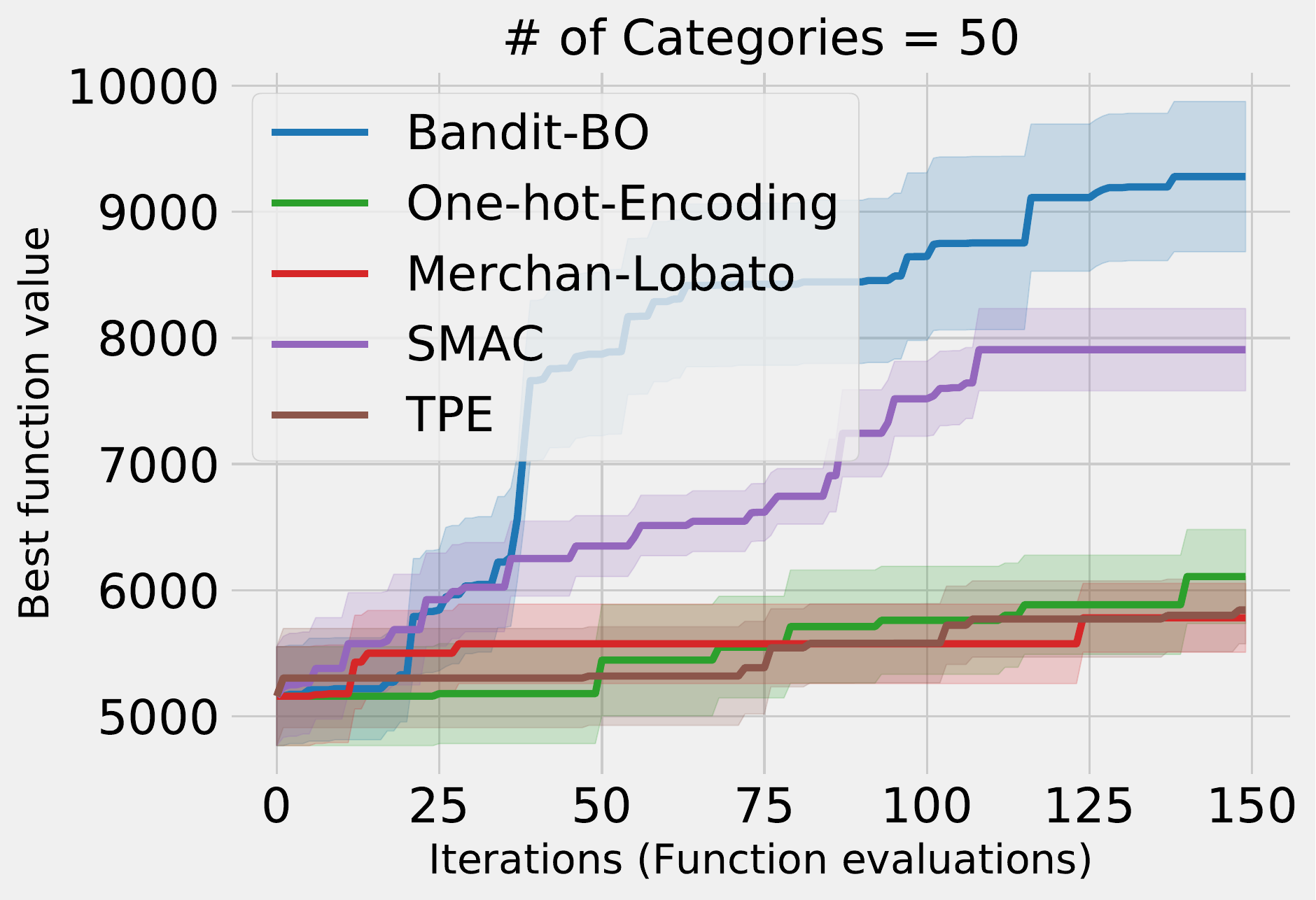}\hspace{0.4cm}\includegraphics[scale=0.25]{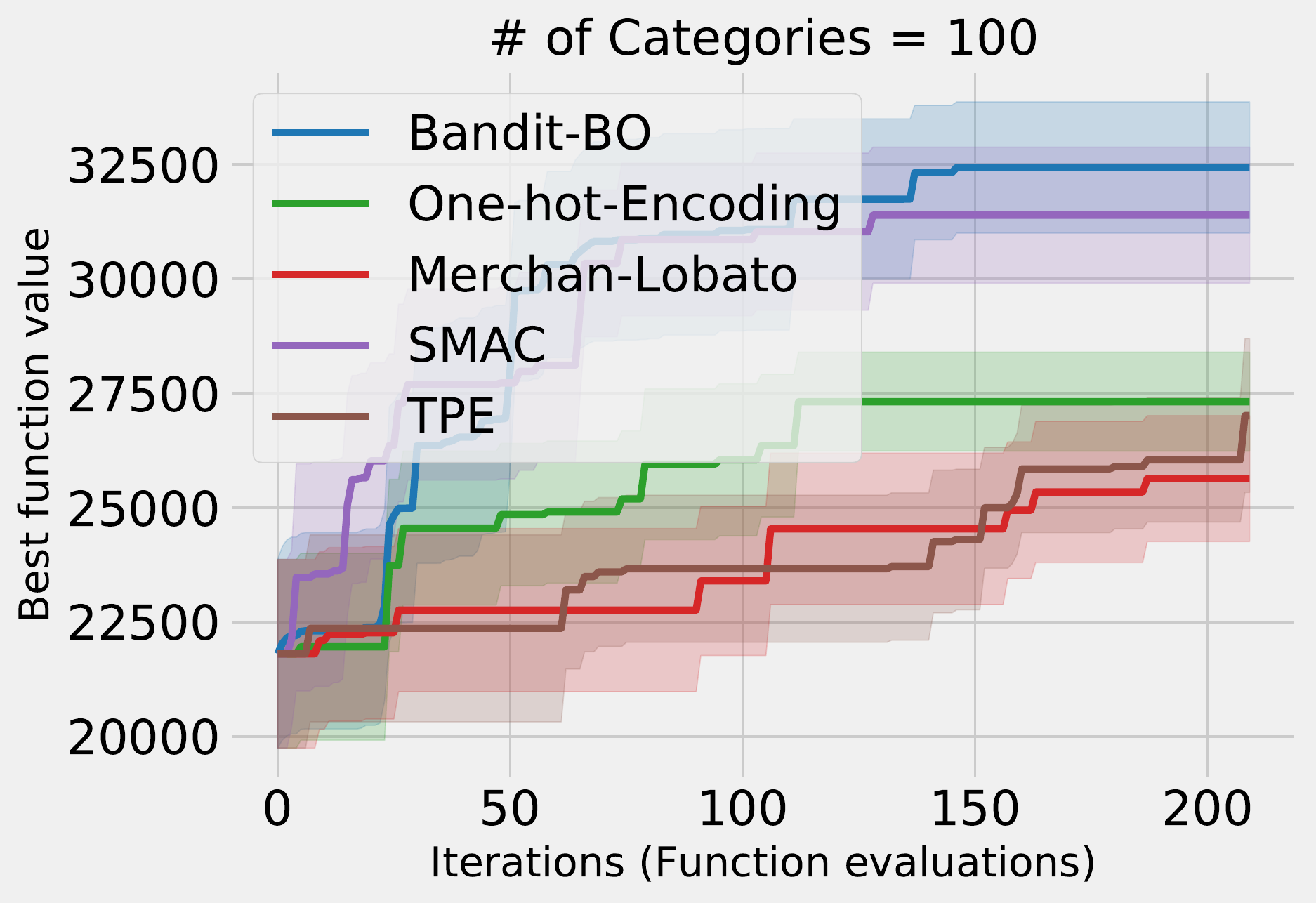}
\par\end{centering}
\caption{\label{fig:Optimization-results-5d-function-c}Optimization results
for the \textit{modified} Alpine-4d function for different numbers
of categories $C\in\{25,50,100\}$. The batch size is fixed to 5.}
\end{figure*}

\subsubsection{5d function (1 categorical variable + 4 continuous variables)}

We also compare our method \textbf{Bandit-BO} with others methods
on a 5d function created by combining a set of \textit{modified} Alpine-4d
functions as follows

\[
f([c,x])={\displaystyle \prod_{i=1}^{4}\sqrt{z_{i}}}\sin(z_{i})+(2\times c),
\]
where $z_{i}=x_{i}+2\times c$ and $x\in[1,10]^{4}$.

Figures \ref{fig:Optimization-results-5d-function-b} and \ref{fig:Optimization-results-5d-function-c}
show the optimization results for the function with different batch
sizes and different numbers of categories respectively. Our \textbf{Bandit-BO}
is the best method while SMAC is the second-best method. TPE and Merchan-Lobato
are generally better than One-hot-Encoding.

\subsection{Experiments with Automated Machine Learning}

\subsubsection{Machine learning models and their hyper-parameters}

Table \ref{tab:Machine-learning-models} summarizes 14 machine learning
models (classifiers) along with their hyper-parameters used in our
experiment \textit{automated machine learning} (see the main paper).
We name the models and their hyper-parameters following the notation
in the Python Machine Learning library \textit{scikit-learn}\footnote{\url{https://scikit-learn.org/}}.
The value range for each hyper-parameter is adopted from the automated
machine learning package \textit{Auto-sklearn} \cite{Feurer2019}.

\begin{table*}[t]
\caption{\label{tab:Machine-learning-models}Machine learning models and hyper-parameters.}

\centering{}%
\begin{tabular}{|l|l|l|}
\hline 
\rowcolor{header_color}\textbf{Model} & \textbf{Hyper-parameter} & \textbf{Value range}\tabularnewline
\hline 
\hline 
\multirow{2}{*}{\textit{Adaboost}} & n\_estimators & $[50,100]$\tabularnewline
\cline{2-3} 
 & learning\_rate & $\log[0.01,2]$\tabularnewline
\hline 
\multirow{3}{*}{\textit{Gradient Boosting}} & learning\_rate & $\log[0.01,1]$\tabularnewline
\cline{2-3} 
 & subsample & $[0.01,1]$\tabularnewline
\cline{2-3} 
 & max\_features & $[0.1,1]$\tabularnewline
\hline 
\textit{Decision Tree} & max\_depth & $[0,2]$\tabularnewline
\hline 
\textit{Extra Trees} & max\_features & $[0,1]$\tabularnewline
\hline 
\multirow{2}{*}{\textit{Random Forest}} & n\_estimators & $[10,50]$\tabularnewline
\cline{2-3} 
 & max\_features & $[0,1]$\tabularnewline
\hline 
\textit{Bernoulli NB} & alpha & $\log[10^{-2},100]$\tabularnewline
\hline 
\textit{Multinomial NB} & alpha & $\log[10^{-2},100]$\tabularnewline
\hline 
\textit{LDA} & shrinkage & $[0,1]$\tabularnewline
\hline 
\textit{QDA} & reg\_param & $[0,1]$\tabularnewline
\hline 
\multicolumn{1}{|l|}{\textit{Linear-SVM}} & C & $\log[2^{-5},2^{15}]$\tabularnewline
\hline 
\multirow{2}{*}{\textit{RBF-SVM}} & C & $\log[2^{-5},2^{15}]$\tabularnewline
\cline{2-3} 
 & gamma & $\log[2^{-15},2^{3}]$\tabularnewline
\hline 
\multicolumn{1}{|l|}{\textit{Passive Aggressive}} & C & $\log[10^{-5},10]$\tabularnewline
\hline 
\multirow{3}{*}{\textit{SGD (logistic loss)}} & alpha & $\log[10^{-7},10^{-1}]$\tabularnewline
\cline{2-3} 
 & l1\_ratio & $\log[10^{-9},1]$\tabularnewline
\cline{2-3} 
 & eta0 & $\log[10^{-7},10^{-1}]$\tabularnewline
\hline 
\multirow{3}{*}{\textit{Neural Network}} & hidden\_layer\_sizes & $\log[128,256]$\tabularnewline
\cline{2-3} 
 & alpha & $\log[10^{-7},10^{-1}]$\tabularnewline
\cline{2-3} 
 & learning\_rate & $\log[10^{-4},10^{-1}]$\tabularnewline
\hline 
\end{tabular}
\end{table*}

\subsubsection{Classification results on all 30 benchmark datasets}

Table \ref{tab:Overall-accuracy} shows the classification accuracy
of each method on 30 benchmark datasets. Compared with other methods,
our method \textbf{Bandit-BO} shows an improved classification accuracy
on most of the datasets.

\begin{table*}
\caption{\label{tab:Overall-accuracy}Characteristics ($|D|$: the number of
samples, $|F|$: the number of features, and $|L|$: the number of
labels) of 30 benchmark datasets along with \textit{classification
accuracy} (standard error) of our method \textbf{Bandit-BO} and other
methods. Bold font marks the best performance in a row. The last row
denotes the overall accuracy of each method across all 30 datasets.}

\centering{}%
\begin{tabular}{|l|l|r|r|r|>{\raggedleft}p{0.1\textwidth}|>{\raggedleft}p{0.1\textwidth}|>{\raggedleft}p{0.1\textwidth}|>{\raggedleft}p{0.1\textwidth}|}
\hline 
\rowcolor{header_color}\textbf{Dataset} & \textbf{Format} & $|D|$ & $|F|$ & $|L|$ & \textbf{Bandit-BO} & \textbf{Hyperopt-sklearn} & \textbf{Auto-sklearn} & \textbf{TPOT}\tabularnewline
\hline 
\hline 
\textit{wine} & tabular & 178 & 13 & 3 & \textbf{98.33 (0.00)} & 97.78 (0.01) & 97.50 (0.01) & 96.67 (0.01)\tabularnewline
\hline 
\rowcolor{even_color}\textit{breast\_cancer} & tabular & 569 & 30 & 2 & \textbf{97.02 (0.01)} & 95.44 (0.00) & 96.40 (0.00) & 96.84 (0.00)\tabularnewline
\hline 
\textit{analcatdata\_authorship} & text & 841 & 70 & 4 & \textbf{99.76 (0.00)} & 99.47 (0.00) & 99.41 (0.00) & 99.53 (0.00)\tabularnewline
\hline 
\rowcolor{even_color}\textit{diabetes} & tabular & 768 & 8 & 2 & \textbf{77.40 (0.01)} & 73.70 (0.02) & 76.95 (0.01) & 77.01 (0.01)\tabularnewline
\hline 
\textit{electricity} & tabular & 45,312 & 8 & 2 & \textbf{92.29 (0.00)} & 92.21 (0.00) & 90.89 (0.00) & 90.94 (0.00)\tabularnewline
\hline 
\rowcolor{even_color}\textit{wall\_robot\_navigation} & trajectory & 5,456 & 24 & 4 & \textbf{99.73 (0.00)} & \textbf{99.73 (0.00)} & 99.43 (0.00) & 99.46 (0.00)\tabularnewline
\hline 
\textit{vehicle} & tabular & 846 & 18 & 4 & \textbf{81.71 (0.01)} & 78.71 (0.01) & 80.24 (0.01) & 78.12 (0.01)\tabularnewline
\hline 
\rowcolor{even_color}\textit{cardiotocography} & tabular & 2,126 & 35 & 10 & \textbf{100.0 (0.00)} & \textbf{100.0 (0.00)} & 99.98 (0.00) & \textbf{100.0 (0.00)}\tabularnewline
\hline 
\textit{artificial\_characters} & text & 10,218 & 7 & 10 & 90.47 (0.01) & \textbf{90.94 (0.00)} & 82.49 (0.00) & 87.75 (0.01)\tabularnewline
\hline 
\rowcolor{even_color}\textit{monks1} & tabular & 556 & 6 & 2 & \textbf{100.0 (0.00)} & 99.82 (0.00) & 99.73 (0.00) & \textbf{100.0 (0.00)}\tabularnewline
\hline 
\textit{monks2} & tabular & 601 & 6 & 2 & 98.26 (0.01) & 97.69 (0.01) & 97.36 (0.01) & \textbf{99.92 (0.00)}\tabularnewline
\hline 
\rowcolor{even_color}\textit{steel\_plates\_fault} & tabular & 1,941 & 33 & 2 & \textbf{100.0 (0.00)} & \textbf{100.0 (0.00)} & \textbf{100.0 (0.00)} & \textbf{100.0 (0.00)}\tabularnewline
\hline 
\textit{phoneme} & tabular & 5,404 & 5 & 2 & \textbf{90.23 (0.00)} & 90.21 (0.00) & 89.25 (0.00) & 89.58 (0.00)\tabularnewline
\hline 
\rowcolor{even_color}\textit{waveform} & tabular & 5,000 & 40 & 3 & \textbf{86.45 (0.00)} & 86.42 (0.00) & 86.19 (0.00) & 86.28 (0.00)\tabularnewline
\hline 
\textit{balance\_scale} & tabular & 625 & 4 & 3 & \textbf{98.48 (0.01)} & 97.20 (0.01) & 89.04 (0.01) & 92.32 (0.01)\tabularnewline
\hline 
\rowcolor{even_color}\textit{digits} & image & 1,797 & 64 & 10 & 98.25 (0.00) & \textbf{98.67 (0.00)} & 98.08 (0.00) & 97.86 (0.00)\tabularnewline
\hline 
\textit{iris} & tabular & 150 & 4 & 3 & 94.33 (0.01) & 92.00 (0.01) & \textbf{95.33 (0.01)} & 94.67 (0.01)\tabularnewline
\hline 
\rowcolor{even_color}\textit{blood\_transfusion} & tabular & 748 & 4 & 2 & 77.87 (0.01) & 76.07 (0.01) & 77.67 (0.01) & \textbf{78.07 (0.01)}\tabularnewline
\hline 
\textit{qsar\_biodeg} & tabular & 1,055 & 41 & 2 & 85.12 (0.01) & 84.79 (0.01) & \textbf{86.54 (0.01)} & 85.97 (0.00)\tabularnewline
\hline 
\rowcolor{even_color}\textit{letter} & image & 20,000 & 16 & 26 & \textbf{97.25 (0.00)} & 97.00 (0.01) & 95.72 (0.00) & 95.72 (0.00)\tabularnewline
\hline 
\textit{australian} & tabular & 690 & 14 & 2 & 83.77 (0.01) & 83.84 (0.01) & 84.13 (0.01) & \textbf{84.28 (0.01)}\tabularnewline
\hline 
\rowcolor{even_color}\textit{olivetti} & image & 400 & 4,096 & 40 & \textbf{99.25 (0.00)} & 96.25 (0.01) & 94.00 (0.01) & 97.22 (0.01)\tabularnewline
\hline 
\textit{spambase} & text & 4,601 & 57 & 2 & 95.10 (0.00) & 95.16 (0.00) & \textbf{95.44 (0.00)} & 94.80 (0.00)\tabularnewline
\hline 
\rowcolor{even_color}\textit{hill\_valley} & graph & 1,212 & 100 & 2 & 75.19 (0.02) & 70.82 (0.04) & \textbf{90.29 (0.01)} & 82.96 (0.03)\tabularnewline
\hline 
\textit{eeg\_eye\_state} & temporal & 14,980 & 14 & 2 & 92.47 (0.00) & 92.34 (0.00) & \textbf{95.10 (0.00)} & 88.70 (0.02)\tabularnewline
\hline 
\rowcolor{even_color}\textit{churn} & tabular & 5,000 & 20 & 2 & 95.76 (0.00) & 95.50 (0.00) & \textbf{95.80 (0.00)} & 95.65 (0.00)\tabularnewline
\hline 
\textit{kc1} & tabular & 2,109 & 21 & 2 & \textbf{86.00 (0.01)} & 85.66 (0.01) & 85.07 (0.01) & \textbf{86.00 (0.01)}\tabularnewline
\hline 
\rowcolor{even_color}\textit{kc2} & tabular & 522 & 21 & 2 & 83.33 (0.01) & 83.05 (0.01) & 82.95 (0.01) & \textbf{83.43 (0.01)}\tabularnewline
\hline 
\textit{segment} & image & 2,310 & 19 & 7 & 94.00 (0.00) & \textbf{94.20 (0.00)} & 94.00 (0.00) & 93.98 (0.00)\tabularnewline
\hline 
\rowcolor{even_color}\textit{gas\_drift} & tabular & 13,910 & 128 & 6 & \textbf{99.60 (0.00)} & 99.58 (0.00) & 99.52 (0.00) & 99.44 (0.00)\tabularnewline
\hline 
\rowcolor{childheader_color}\textbf{Average} &  &  &  &  & \textbf{92.25 (0.00)} & 91.48 (0.01) & 91.82 (0.00) & 91.77 (0.01)\tabularnewline
\hline 
\end{tabular}
\end{table*}

\end{document}

%% file: macros.tex
\global\long\def\se{\hat{\text{se}}}

\global\long\def\interior{\text{int}}

\global\long\def\boundary{\text{bd}}

\global\long\def\new{\text{*}}

\global\long\def\stir{\text{Stirl}}

\global\long\def\dist{d}

\global\long\def\HX{\entro\left(X\right)}
 \global\long\def\entropyX{\HX}

\global\long\def\HY{\entro\left(Y\right)}
 \global\long\def\entropyY{\HY}

\global\long\def\HXY{\entro\left(X,Y\right)}
 \global\long\def\entropyXY{\HXY}

\global\long\def\mutualXY{\mutual\left(X;Y\right)}
 \global\long\def\mutinfoXY{\mutualXY}

\global\long\def\xnew{y}

\global\long\def\bx{\boldsymbol{x}}

\global\long\def\bs{\boldsymbol{s}}

\global\long\def\bk{\boldsymbol{k}}

\global\long\def\bX{\boldsymbol{X}}

\global\long\def\tbx{\tilde{\bx}}

\global\long\def\by{\mathbf{y}}

\global\long\def\bY{\boldsymbol{Y}}

\global\long\def\bZ{\boldsymbol{Z}}

\global\long\def\bU{\boldsymbol{U}}

\global\long\def\bv{\boldsymbol{v}}

\global\long\def\bn{\boldsymbol{n}}

\global\long\def\bV{\boldsymbol{V}}

\global\long\def\bK{\boldsymbol{K}}

\global\long\def\bw{\vt w}

\global\long\def\bbeta{\gvt{\beta}}

\global\long\def\bmu{\gvt{\mu}}

\global\long\def\btheta{\boldsymbol{\theta}}

\global\long\def\blambda{\boldsymbol{\lambda}}

\global\long\def\bgamma{\boldsymbol{\gamma}}

\global\long\def\bpsi{\boldsymbol{\psi}}

\global\long\def\bphi{\boldsymbol{\phi}}

\global\long\def\bpi{\boldsymbol{\pi}}

\global\long\def\eeta{\boldsymbol{\eta}}

\global\long\def\bomega{\boldsymbol{\omega}}

\global\long\def\bepsilon{\boldsymbol{\epsilon}}

\global\long\def\btau{\boldsymbol{\tau}}

\global\long\def\bSigma{\gvt{\Sigma}}

\global\long\def\realset{\mathbb{R}}

\global\long\def\realn{\realset^{n}}

\global\long\def\integerset{\mathbb{Z}}

\global\long\def\natset{\integerset}

\global\long\def\integer{\integerset}

\global\long\def\natn{\natset^{n}}

\global\long\def\rational{\mathbb{Q}}

\global\long\def\rationaln{\rational^{n}}

\global\long\def\complexset{\mathbb{C}}

\global\long\def\comp{\complexset}

\global\long\def\compl#1{#1^{\text{c}}}

\global\long\def\and{\cap}

\global\long\def\compn{\comp^{n}}

\global\long\def\comb#1#2{\left({#1\atop #2}\right) }

\global\long\def\nchoosek#1#2{\left({#1\atop #2}\right)}

\global\long\def\param{\vt w}

\global\long\def\Param{\Theta}

\global\long\def\meanparam{\gvt{\mu}}

\global\long\def\Meanparam{\mathcal{M}}

\global\long\def\meanmap{\mathbf{m}}

\global\long\def\logpart{A}

\global\long\def\simplex{\Delta}

\global\long\def\simplexn{\simplex^{n}}

\global\long\def\dirproc{\text{DP}}

\global\long\def\ggproc{\text{GG}}

\global\long\def\DP{\text{DP}}

\global\long\def\ndp{\text{nDP}}

\global\long\def\hdp{\text{HDP}}

\global\long\def\gempdf{\text{GEM}}

\global\long\def\ei{\text{EI}}

\global\long\def\rfs{\text{RFS}}

\global\long\def\bernrfs{\text{BernoulliRFS}}

\global\long\def\poissrfs{\text{PoissonRFS}}

\global\long\def\grad{\gradient}
 \global\long\def\gradient{\nabla}

\global\long\def\cpr#1#2{\Pr\left(#1\ |\ #2\right)}

\global\long\def\var{\text{Var}}

\global\long\def\Var#1{\text{Var}\left[#1\right]}

\global\long\def\cov{\text{Cov}}

\global\long\def\Cov#1{\cov\left[ #1 \right]}

\global\long\def\COV#1#2{\underset{#2}{\cov}\left[ #1 \right]}

\global\long\def\corr{\text{Corr}}

\global\long\def\sst{\text{T}}

\global\long\def\SST{\sst}

\global\long\def\ess{\mathbb{E}}

\global\long\def\Ess#1{\ess\left[#1\right]}

\newcommandx\ESS[2][usedefault, addprefix=\global, 1=]{\underset{#2}{\ess}\left[#1\right]}

\global\long\def\fisher{\mathcal{F}}

\global\long\def\bfield{\mathcal{B}}
 \global\long\def\borel{\mathcal{B}}

\global\long\def\bernpdf{\text{Bernoulli}}

\global\long\def\betapdf{\text{Beta}}

\global\long\def\dirpdf{\text{Dir}}

\global\long\def\gammapdf{\text{Gamma}}

\global\long\def\gaussden#1#2{\text{Normal}\left(#1, #2 \right) }

\global\long\def\gauss{\mathbf{N}}

\global\long\def\gausspdf#1#2#3{\text{Normal}\left( #1 \lcabra{#2, #3}\right) }

\global\long\def\multpdf{\text{Mult}}

\global\long\def\poiss{\text{Pois}}

\global\long\def\poissonpdf{\text{Poisson}}

\global\long\def\pgpdf{\text{PG}}

\global\long\def\wshpdf{\text{Wish}}

\global\long\def\iwshpdf{\text{InvWish}}

\global\long\def\nwpdf{\text{NW}}

\global\long\def\niwpdf{\text{NIW}}

\global\long\def\studentpdf{\text{Student}}

\global\long\def\unipdf{\text{Uni}}

\global\long\def\transp#1{\transpose{#1}}
 \global\long\def\transpose#1{#1^{\mathsf{T}}}

\global\long\def\mgt{\succ}

\global\long\def\mge{\succeq}

\global\long\def\idenmat{\mathbf{I}}

\global\long\def\trace{\mathrm{tr}}

\global\long\def\argmax#1{\underset{_{#1}}{\text{argmax}} }

\global\long\def\argmin#1{\underset{_{#1}}{\text{argmin}\ } }

\global\long\def\diag{\text{diag}}

\global\long\def\norm{}

\global\long\def\spn{\text{span}}

\global\long\def\vtspace{\mathcal{V}}

\global\long\def\field{\mathcal{F}}
 \global\long\def\ffield{\mathcal{F}}

\global\long\def\inner#1#2{\left\langle #1,#2\right\rangle }
 \global\long\def\iprod#1#2{\inner{#1}{#2}}

\global\long\def\dprod#1#2{#1 \cdot#2}

\global\long\def\norm#1{\left\Vert #1\right\Vert }

\global\long\def\entro{\mathbb{H}}

\global\long\def\entropy{\mathbb{H}}

\global\long\def\Entro#1{\entro\left[#1\right]}

\global\long\def\Entropy#1{\Entro{#1}}

\global\long\def\mutinfo{\mathbb{I}}

\global\long\def\relH{\mathit{D}}

\global\long\def\reldiv#1#2{\relH\left(#1||#2\right)}

\global\long\def\KL{KL}

\global\long\def\KLdiv#1#2{\KL\left(#1\parallel#2\right)}
 \global\long\def\KLdivergence#1#2{\KL\left(#1\ \parallel\ #2\right)}

\global\long\def\crossH{\mathcal{C}}
 \global\long\def\crossentropy{\mathcal{C}}

\global\long\def\crossHxy#1#2{\crossentropy\left(#1\parallel#2\right)}

\global\long\def\breg{\text{BD}}

\global\long\def\lcabra#1{\left|#1\right.}

\global\long\def\lbra#1{\lcabra{#1}}

\global\long\def\rcabra#1{\left.#1\right|}

\global\long\def\rbra#1{\rcabra{#1}}

%% file: introduction.tex
Bayesian optimization (BO) \cite{shahriari2016taking} provides a
powerful and efficient framework for global optimization of expensive
black-box functions. Typically, at each iteration a BO method first
models the black-box function via a statistical model (e.g. a Gaussian
process (GP)) and then seeks out the next function evaluation points
by maximizing an easy to optimize function (a.k.a. \textit{acquisition
function}) that balances the two conflicting requirements: exploitation
of current function knowledge and exploration to gain more function
knowledge. A notable strength of BO is that its convergence is well
studied \cite{bull2011convergence}. 

Most BO methods assume that the function inputs are continuous variables.
In reality, however, a function may be defined over diverse input
types -- for example, categorical, integer or continuous. Categorical
type variables are particularly challenging since they do not have
a natural ordering as in integer and continuous variables. Limited
work has addressed incorporation of categorical input types. In a
recent work, \cite{golovin2017google} used \textit{one-hot encoding}
for categorical variables, increasing the input dimension by one extra
variable per category. Since the categories are mutually exclusive,
all the extra variables are set to zero except the \textit{active}
variable (corresponding to the required category) which is set to
one. After converting the categorical variables to one-hot encoding,
this approach treats extra variables as continuous in $[0,1]$ and
uses a typical BO algorithm to optimize them. Unfortunately, this
type of encoding imposes equal measure of covariance between all category
pairs, totally ignoring the fact that they may have different or no
correlations at all. Further, the recommendations can get repeated
as they are generated via rounding-off at the end. To address the
latter problem, \cite{garrido2018dealing} assumed that the objective
function does not change its values except at the designated points
of 0 and 1. This is achieved by using a kernel function that computes
covariances after the input is rounded off. However, this makes the
resulting acquisition function step-wise, which is hard to optimize. 

In many optimization problems, an additional challenge arises --
each category is coupled with a different continuous search space.
For example, consider the problem of \textit{automated machine learning}
\cite{Feurer2019} where we need to automatically select the best
performing machine learning model along with its optimal hyper-parameters.
Each machine learning model can be viewed as a distinct value (or
choice) of a \emph{categorical variable} while the hyper-parameters
of the model can be viewed as category-specific \textit{continuous}
\emph{variables}. None of the current GP-based BO methods can be applied
to such complex search spaces. 

To incorporate categorical inputs and deal with category-specific
continuous search spaces, tree-based methods have been proposed. For
example, SMAC tackles this problem by using random forest in place
of GP \cite{hutter2011sequential}. However, random forest has a well-known
limitation in performing extrapolation and thus is not a good choice
for BO \cite{Lakshminarayanan_etal_16mondrian}. Yet another method,
tree-structured parzen window based approach (TPE) \cite{bergstra2011algorithms}
can naturally cope with both categorical and continuous variables.
In contrast to GP based approaches, TPE models two likelihood functions
to assess if the function value at any point would be in the range
of top few observations or not. The disadvantage is that this approach
requires higher number of initial data points to model the likelihood
functions effectively. Moreover, the overarching problem with all
the above-mentioned methods is that none of them offered an avenue
for convergence analysis and thus remained ad-hoc in nature.

In this work, we develop a BO algorithm that can handle both categorical
and continuous variables even if each category involves a different
set of continuous variables. The algorithm is amenable to convergence
analysis. We extend this algorithm to develop a \textit{batch} BO
algorithm for the same complex search space scenario. We use a mix
of multi-armed bandit (MAB) and BO formulations in our approach. Each
categorical value corresponds to an arm with its reward distribution
centered around the optimum of the objective function in continuous
variables. Thompson sampling (TS) \cite{russo2014learning} is used
for both selecting the best arm and suggesting the next continuous
point to evaluate. Among various possibilities of MAB and BO algorithms,
our choice of TS is guided by the need for an unified framework to
join both the MAB and BO seamlessly into a single entity for convergence
analysis. Empirically also TS is known to be a competitive algorithm
for both MAB and BO problems due to using a range of exploitation/exploration
trade-offs \cite{russo2014learning,kandasamy2018parallelised}. We
derive the theoretical regret bounds for both the sequential and batch
algorithms, and show that the growth of cumulative regret is at most
sub-linear. We perform a variety of experiments: optimization of several
benchmark functions, hyper-parameter tuning of a neural network, and
\textit{automated machine learning}. The empirical results demonstrate
the effectiveness of our method.

Compared with other methods, our method offers three key advantages:
\begin{itemize}
\item Optimizing black-box functions with continuous and categorical inputs
in both sequential and batch settings;
\item Handling the problems where each category is coupled with a different
set of continuous variables; and
\item Deriving the regret bounds for both sequential and batch settings.
\end{itemize}

%% file: relatedwork.tex
\subsection{Bayesian Optimization}

Bayesian optimization (BO) is a method to find the global optimum
of an expensive, black-box function $f(x)$ as $x^{*}=\text{argmax}_{x\in\mathcal{X}}f(x)$,
where ${\cal X}$ is a bounded domain in $\mathbb{R}^{d}$. It assumes
that $f(x)$ can only be noisily evaluated through queries to the
black-box. BO operates sequentially, and the next function evaluation
is guided by the previous observations. Let us assume that the observations
up to iteration $t$ are denoted as ${\cal D}_{t}=\{x_{i},y_{i}\}_{i=1}^{t}$,
where $y_{i}=f(x_{i})+\epsilon_{i}$ and $\epsilon_{i}\sim{\cal N}(0,\sigma_{\epsilon}^{2})$.
Typically, $f(x)$ is assumed to be a smooth function and modeled
using a Gaussian process (GP), \emph{i.e.} $f(x)\sim\text{GP}(m(x),k(x,x'))$,
where $m(x)$ is a mean function that can be assumed to be a zero
function, and $k(x,x')$ is a covariance function modeling the covariance
between any two function values $f(x)$ and $f(x')$. A common covariance
function is the \textit{squared exponential} kernel, defined as $k(x,x')=\sigma^{2}\exp(-\frac{1}{2l^{2}}||x-x'||_{2}^{2})$,
where $\sigma^{2}$ is a parameter dictating the uncertainty in $f(x)$,
$l$ is a length scale parameter. The \textit{predictive distribution}
for $f(x)$ at any point $x$ is also a Gaussian distribution with
its mean and variance given by $\mu_{t}(x)=\transp{\boldsymbol{k}}[K+\sigma_{\epsilon}^{2}I]^{-1}y_{1:t}$
and $\sigma_{t}^{2}(x)=k(x,x)-\transp{\boldsymbol{k}}[K+\sigma_{\epsilon}^{2}I]^{-1}\boldsymbol{k}$,
where $K$ is a matrix of size $t\times t$ with $(i,j)$-th element
defined as $k(x_{i},x_{j})$ and $\boldsymbol{k}$ is a vector with
$i$-th element defined as $k(x_{i},x)$.

BO uses a surrogate function called \textit{acquisition function}
to find the next point to evaluate. The acquisition function uses
the predictive distribution to balance two contrasting goals: sampling
where the function is expected to take a high value vs. sampling where
the uncertainty about the function value is high. Some well-known
acquisition functions are probability of improvement \cite{kushner1964new},
expected improvement \cite{jones1998efficient}, GP-upper confidence
bound \cite{srinivas2012information}, and predictive entropy search
\cite{hernandez2014predictive}.

\subsection{Batch Bayesian Optimization}

In its standard form, BO works in a \textit{sequential setting} where
it suggests one point at each iteration. However, when parallel resources
are available at each iteration, BO is extended to a \textit{batch
setting} where it suggests multiple points at each iteration. Several
batch BO methods have been developed (see \cite{gonzalez2016batch}
for a review). Recently, Thompson sampling has become an efficient
technique for batch BO \cite{hernandez2017parallel,kandasamy2018parallelised};
the idea is to select a batch element by maximizing a randomly drawn
function from the posterior GP. However, to the best of our knowledge,
batch BO for both categorical and continuous variables has not yet
been studied.

%% file: framework_part1.tex
\subsection{Problem Definition}

We present a method for BO to jointly handle both categorical and
category-specific continuous input variables. Formally, given an input
$[c,x_{c}]$ and a \textit{black-box} function $f([c,x_{c}])$, where
$c\in\{1,...,C\}$ is a \textit{category }among $C$ categories and
$x_{c}\in\mathcal{X}_{c}\subset\mathbb{R}^{d}$ are \textit{continuous}
variables corresponding to the category $c$, our goal is to find:
\begin{equation}
[c^{*},x^{*}]=\argmax{c\in\{1,...,C\},x_{c}\in\mathcal{X}_{c}}f([c,x_{c}])\label{eq:Optimization-problem-Cat-Con}
\end{equation}
A na\"{i}ve approach to solve Eq. (\ref{eq:Optimization-problem-Cat-Con})
is to consider $f([c,x_{c}])$ as a \textit{collection of black-box
functions} $\{f_{c}(x_{c})\}_{c=1}^{C}$ defined on continuous domain
$\mathcal{X}_{c}$ and then find the optimum $x_{c}^{*}=\text{argmax}_{x_{c}\in\mathcal{X}_{c}}f_{c}(x)$
for each function and finally get $[c^{*},x^{*}]$ as $c^{*}=\text{argmax}_{c\in\{1,...,C\}}f_{c}(x_{c}^{*})$
and $x^{*}=x_{c^{*}}^{*}$. However, this approach is \emph{inefficient}
as it needs to find $x_{c}^{*}$ for all $c$. Since the function
evaluations in BO are expensive, an efficient approach is needed.

Before proceeding further, we simplify our notation. Instead of using
$[c,x_{c}]$, we simply write $[c,x]$. The notation $x\in\mathcal{X}_{c}$
is sufficient to resolve any ambiguity.

\subsection{Sequential Setting\label{subsec:Sequential-Bandit-BO}}

To efficiently solve the optimization problem in Eq. (\ref{eq:Optimization-problem-Cat-Con}),
we propose a novel method that combines multi-armed bandit (MAB) and
BO. The main idea is that instead of optimizing $f([c,x])$ exhaustively
for each value of $c$, we attempt to identify (in parallel) the best
value of $c$ (denoted as $c^{*}$), which results in the optimal
function value and optimize $x^{*}=\text{argmax}_{x\in\mathcal{X}_{c^{*}}}f([c^{*},x])$.

We formulate the problem of selecting the best value $c^{*}$ as a
MAB problem \cite{bubeck2012regret}, where each value $c\in\{1,...,C\}$
is considered as an arm. In the usual parlance of MAB, the arm $c$
has a reward distribution (due to noisy function evaluations) around
a mean parameter $f_{c}^{*}\triangleq\max_{x\in\mathcal{X}_{c}}f_{c}(x)$.
We note that $f_{c}^{*}$ is unknown due to $f_{c}(x)$ being a black-box.
For BO of $f_{c}(x)$, we model it using a GP, which induces a distribution
on $f_{c}^{*}$. As we increasingly get observations of $f_{c}(x)$,
the uncertainty in $f_{c}^{*}$ reduces. From the MAB side, this means
that each time an arm $c$ is played, we get an additional observation
of $f_{c}(x)$, which improves the estimate of the mean parameter
$f_{c}^{*}$. MAB algorithm allows us to play the arms optimally guided
by the uncertainties in $f_{c}^{*}$. An illustration is shown in
Figure \ref{fig:An-Illustration-of-Bandit-BO}.

\begin{figure}[t]
\subfloat[Iteration $t=4$.]{\begin{centering}
\begin{minipage}[t]{0.49\textwidth}%
\begin{center}
\includegraphics[scale=0.22]{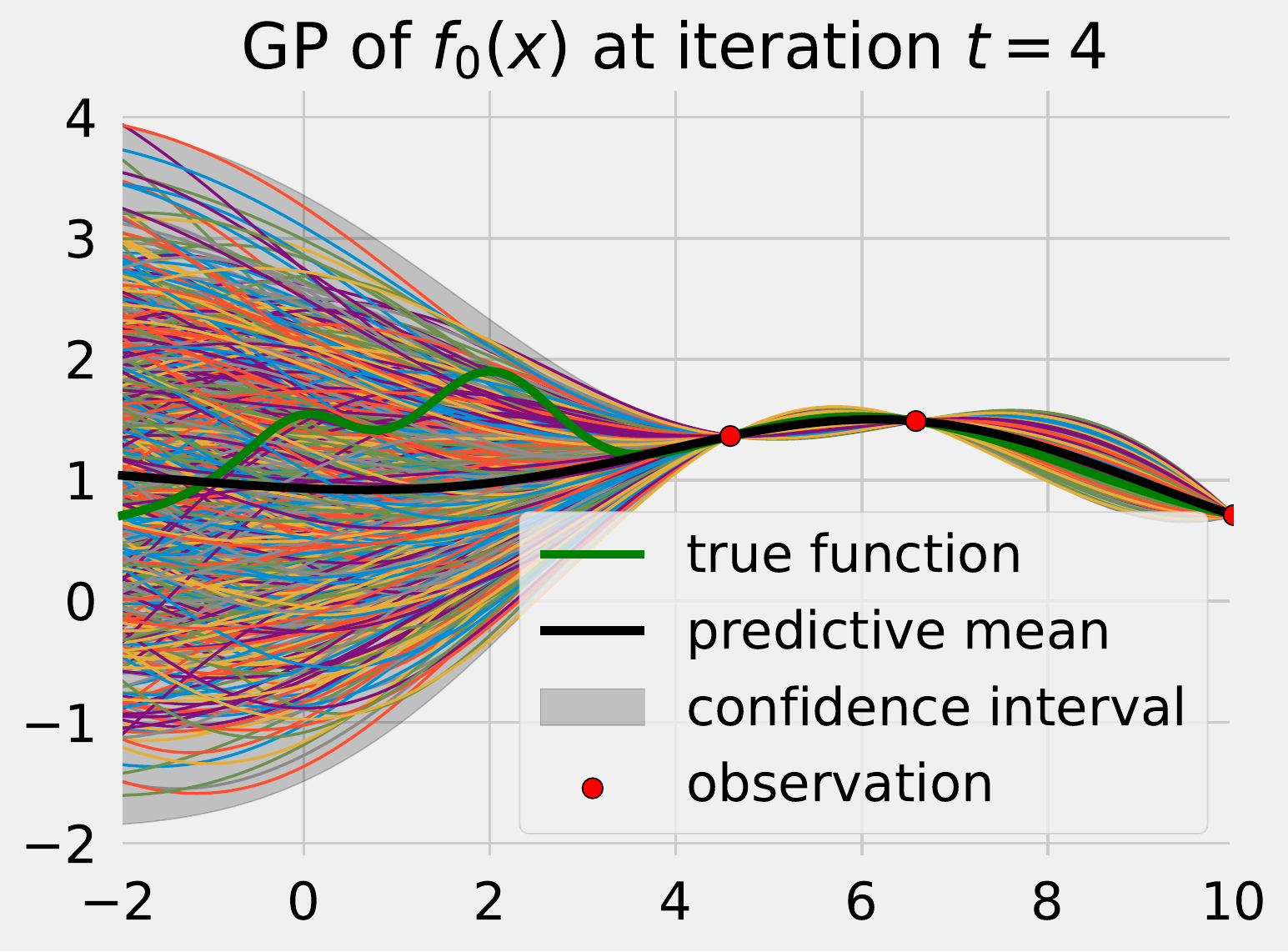}\hspace{0.4cm}\includegraphics[scale=0.21]{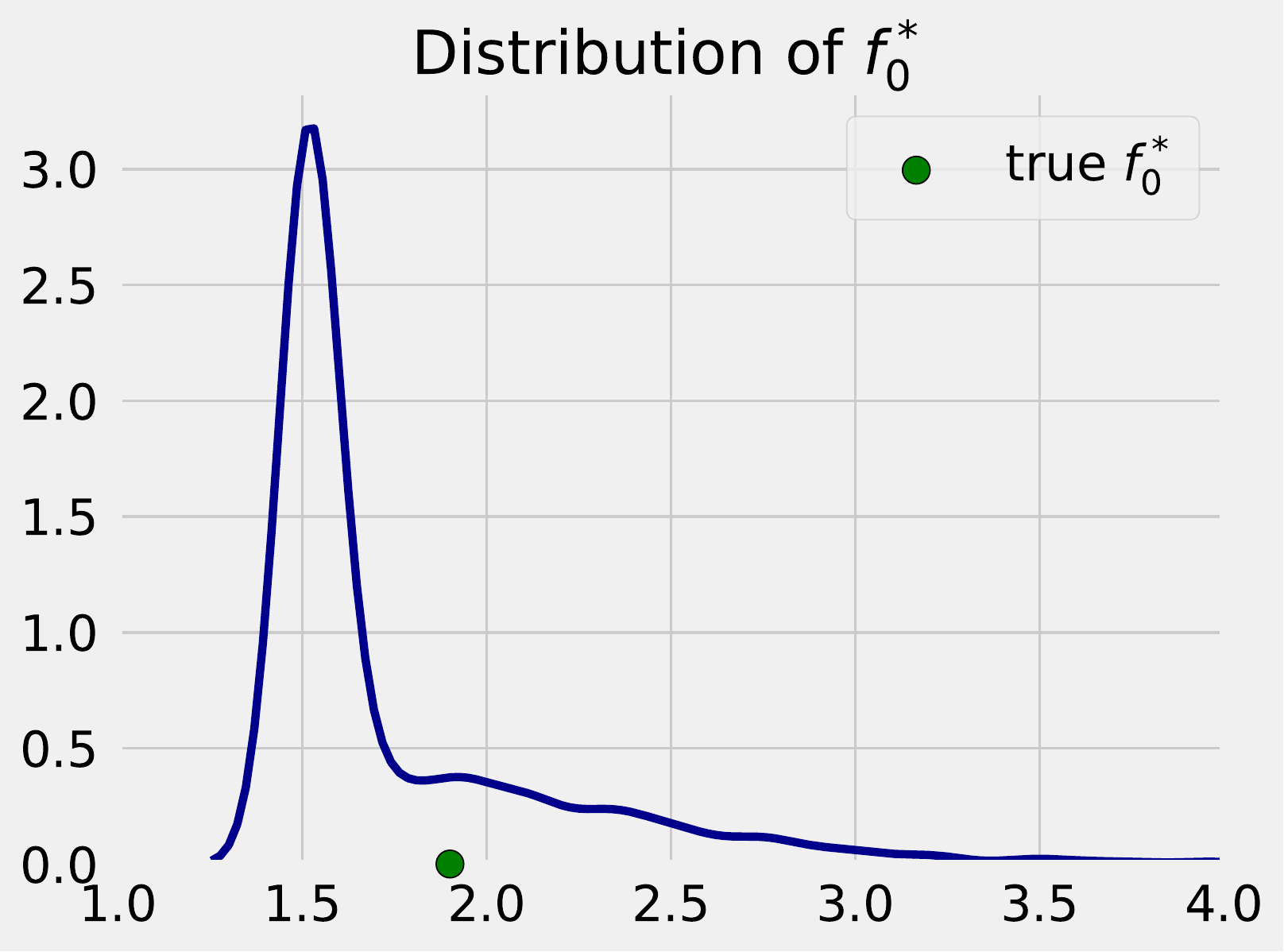}
\par\end{center}
\vspace{0.1cm}

\begin{center}
\includegraphics[scale=0.22]{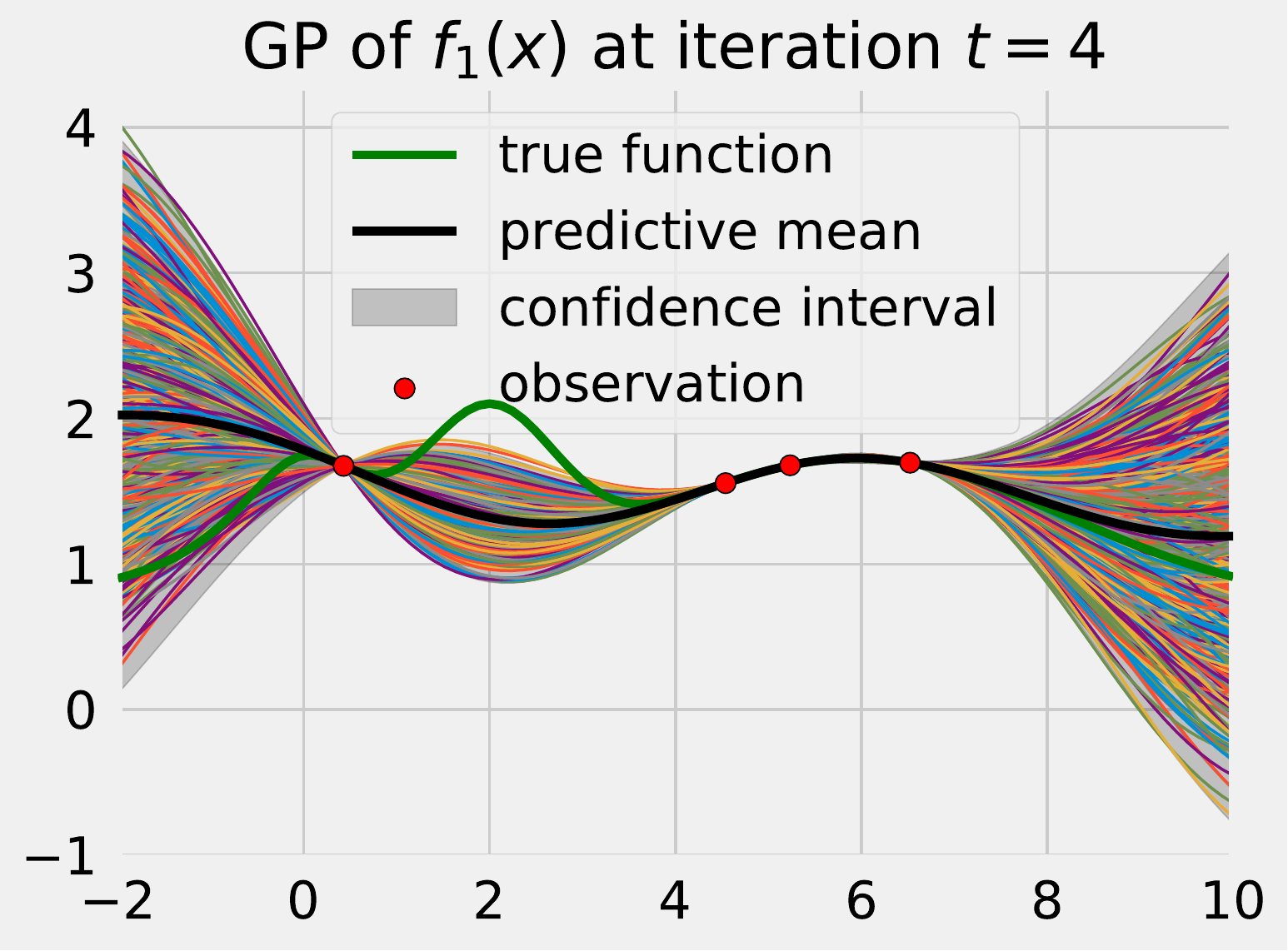}\hspace{0.4cm}\includegraphics[scale=0.21]{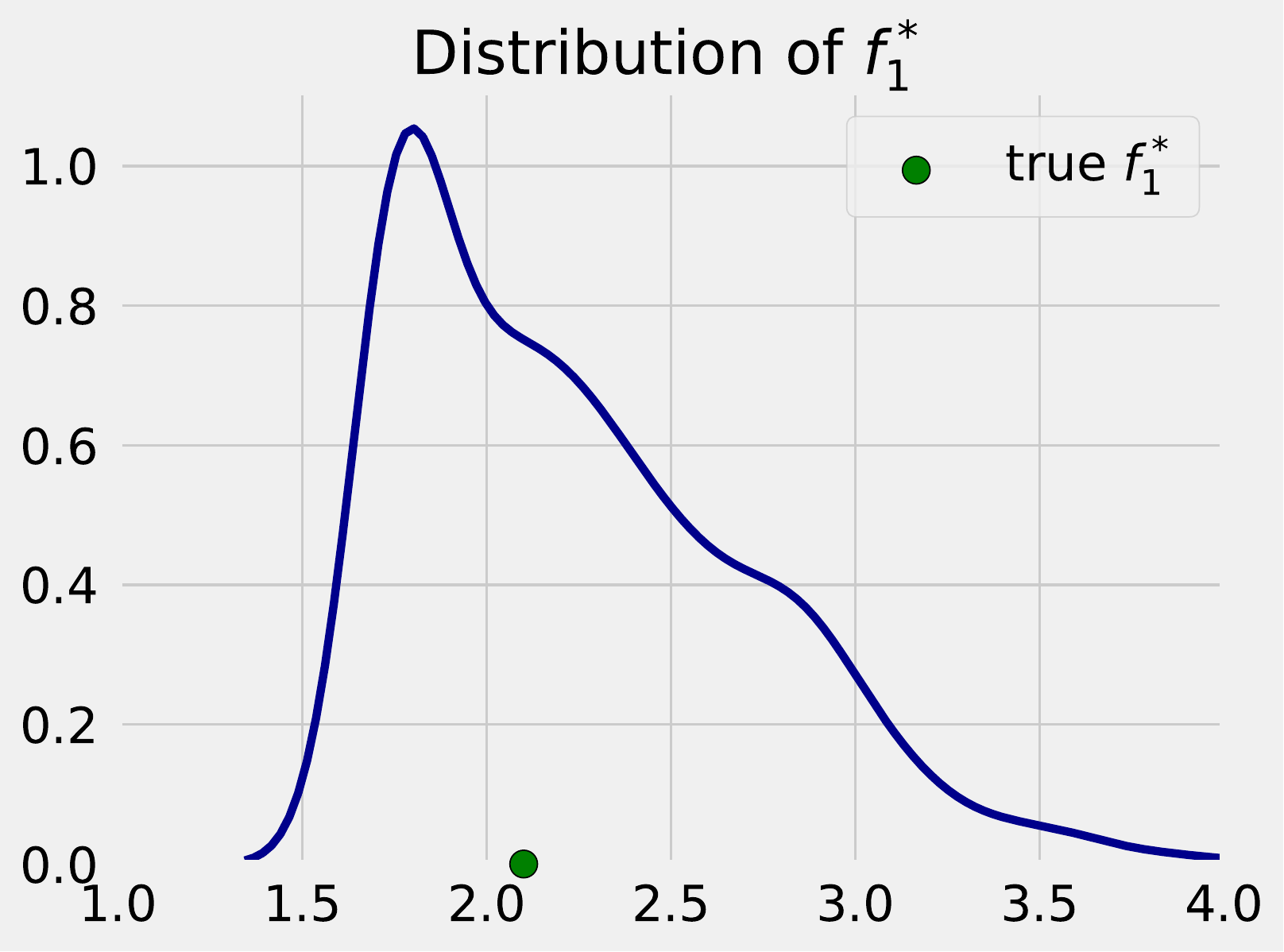}
\par\end{center}%
\end{minipage}
\par\end{centering}
}\hfill{}\subfloat[Iteration $t=15$.]{\begin{centering}
\begin{minipage}[t]{0.49\textwidth}%
\begin{center}
\includegraphics[scale=0.22]{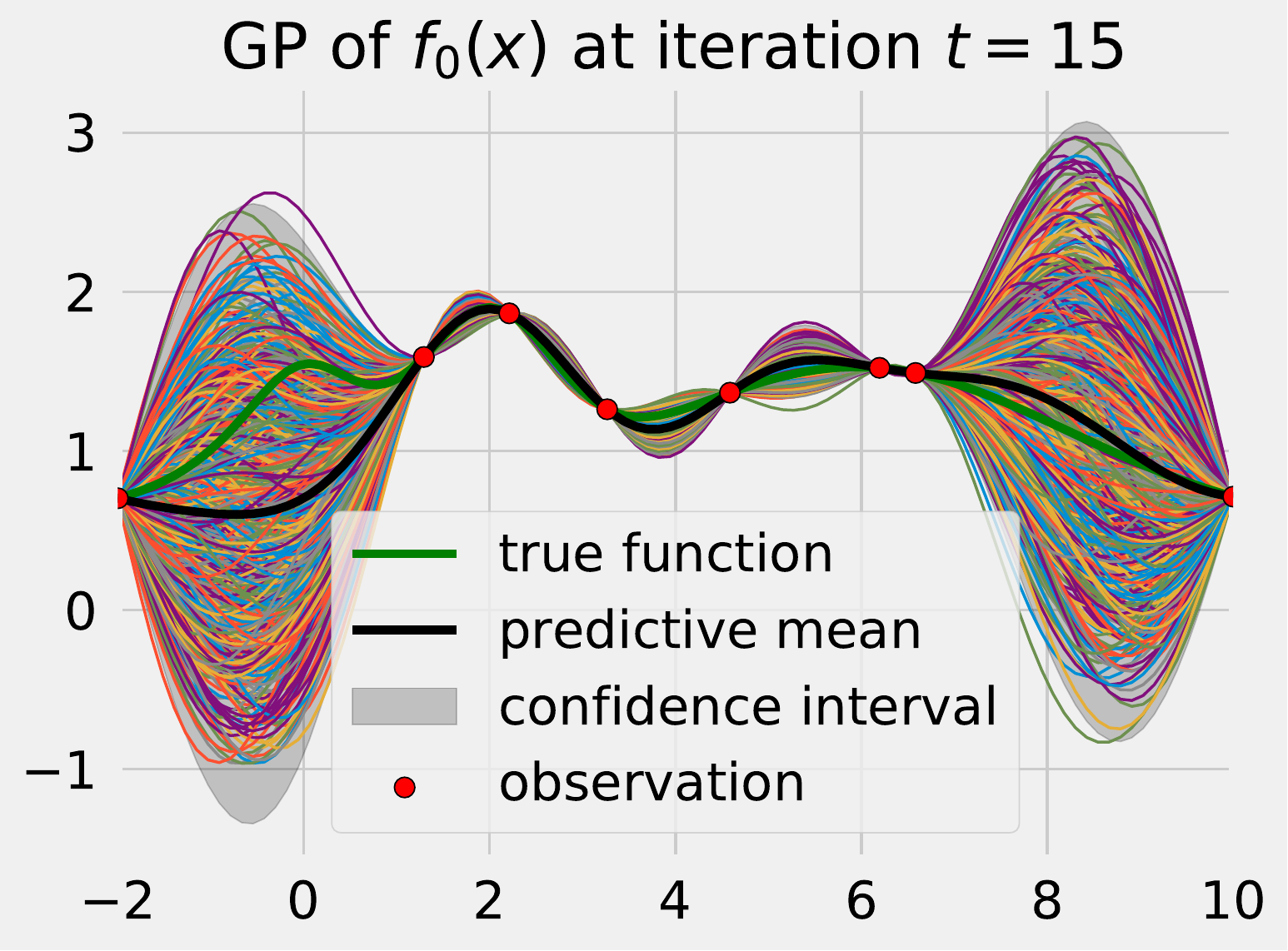}\hspace{0.4cm}\includegraphics[scale=0.21]{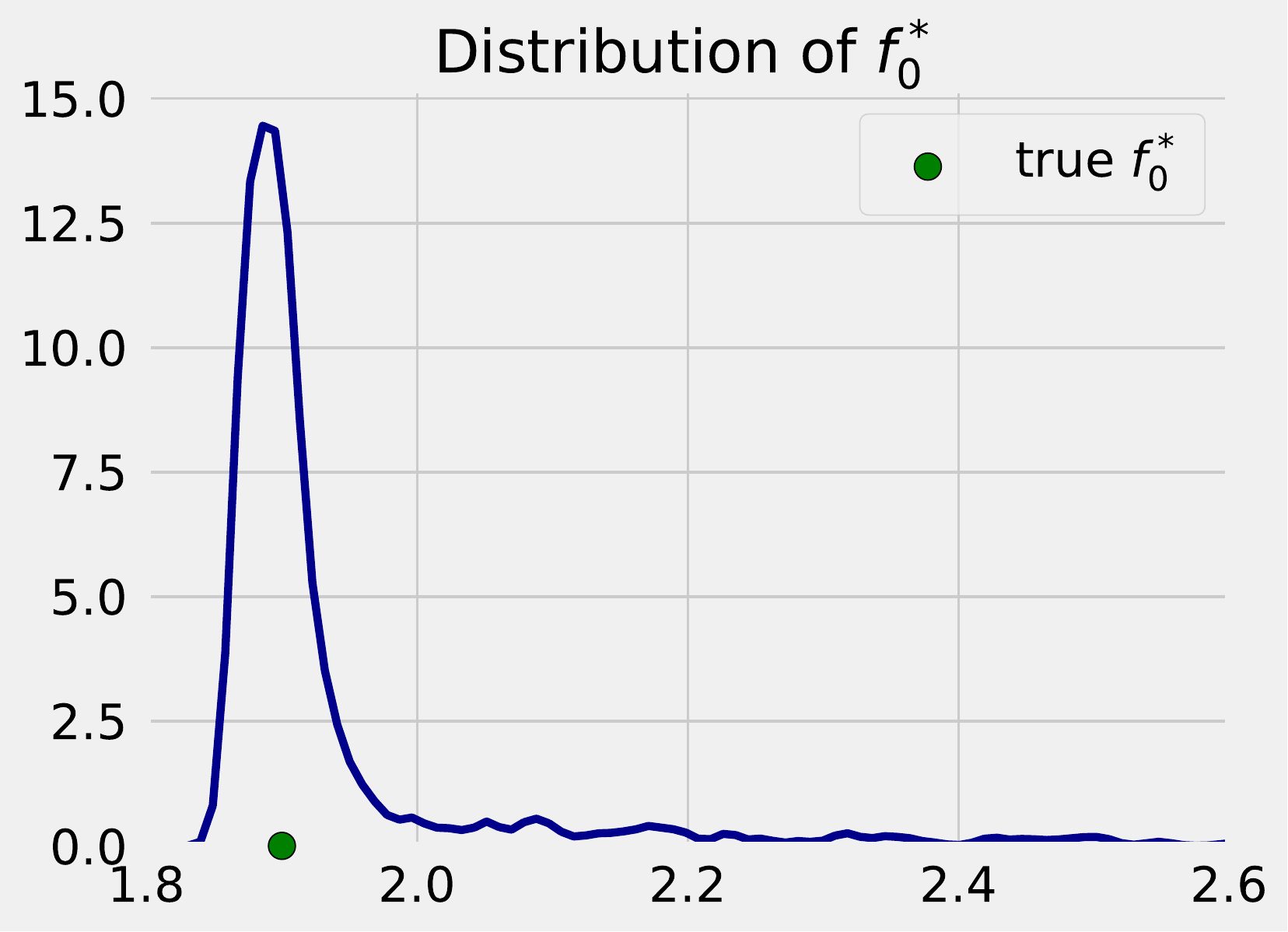}
\par\end{center}
\vspace{0.1cm}

\begin{center}
\includegraphics[scale=0.22]{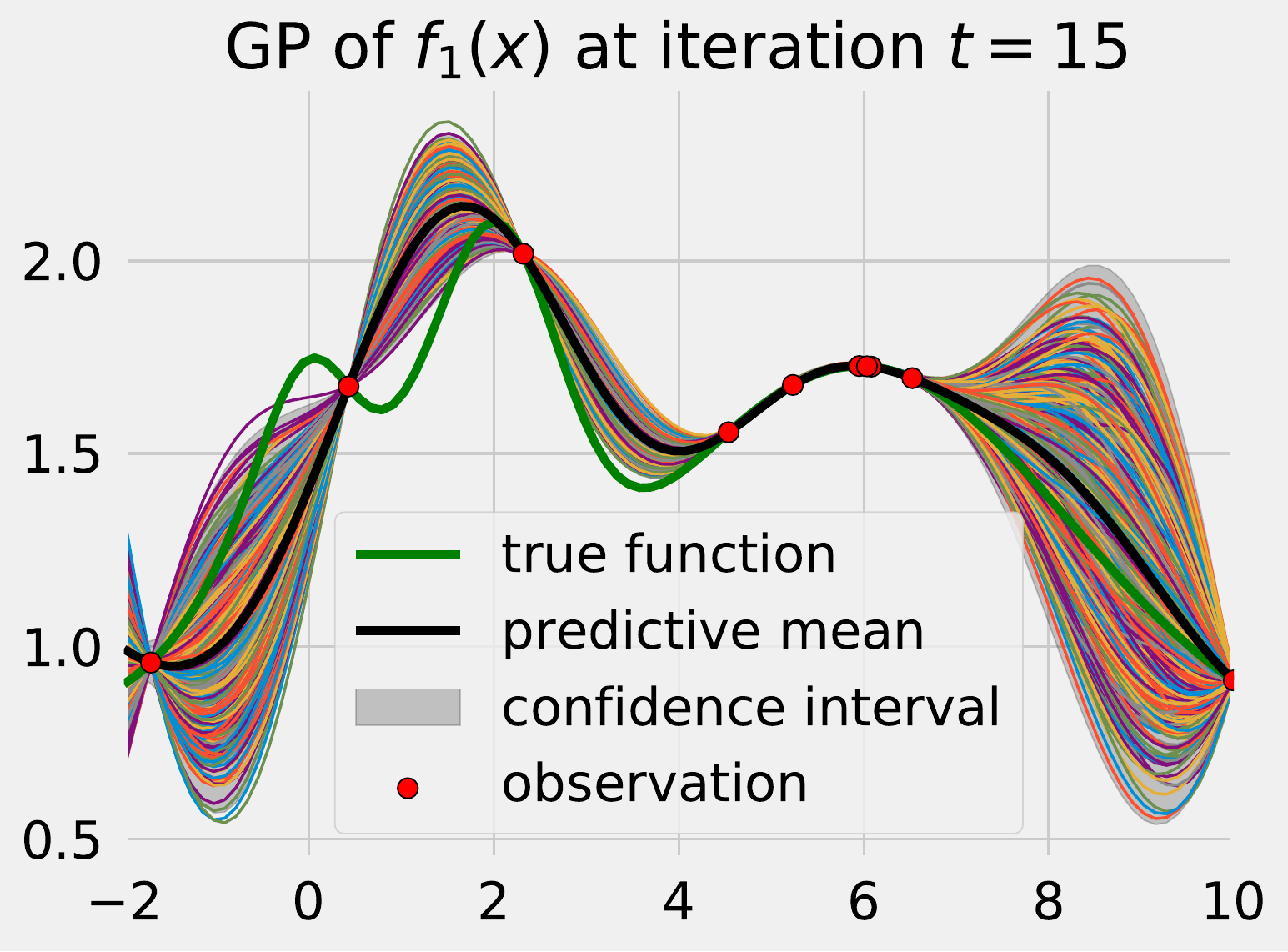}\hspace{0.4cm}\includegraphics[scale=0.21]{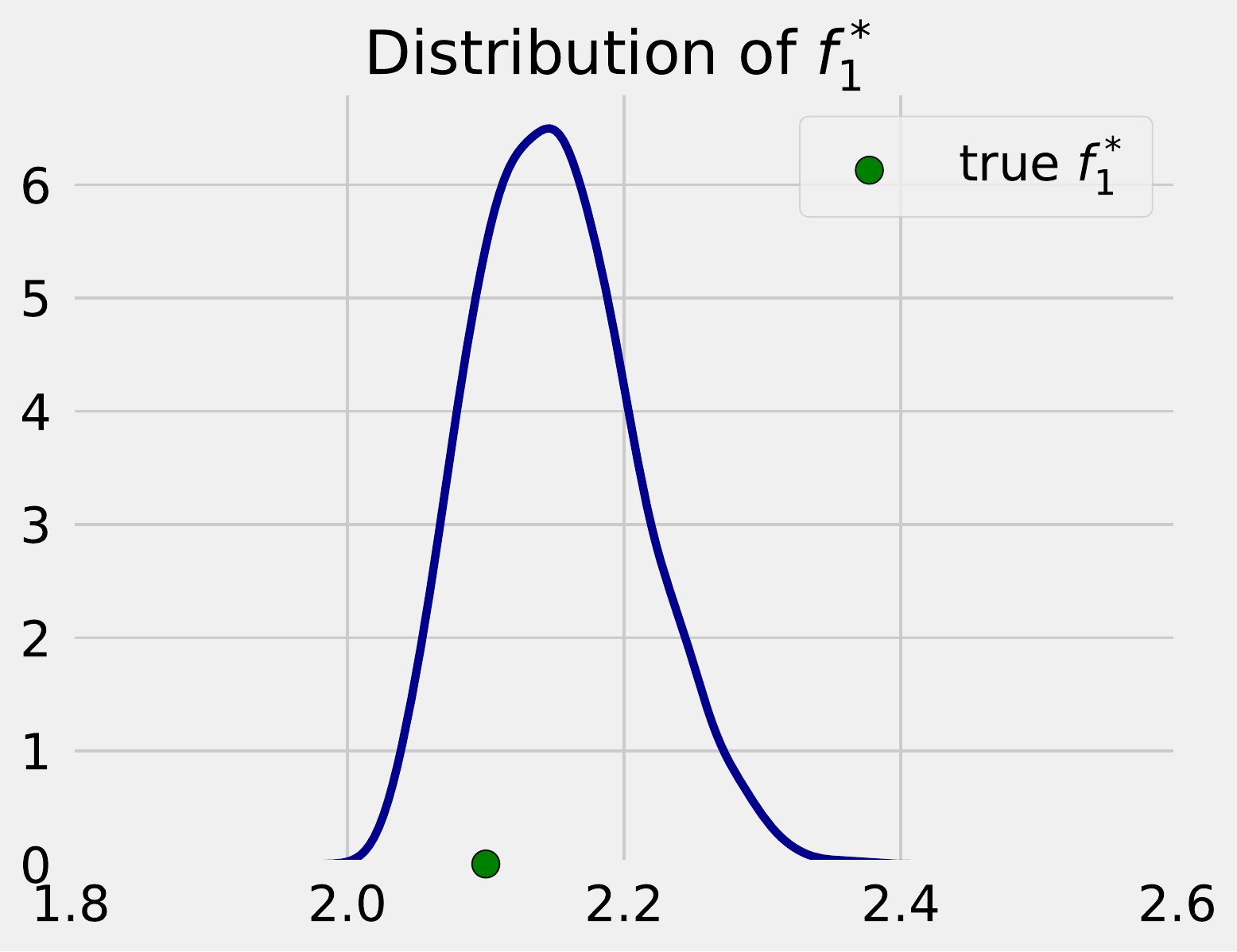}
\par\end{center}%
\end{minipage}
\par\end{centering}
}

\caption{An illustration of our method using $C=2$: (a) the results at $t=4$,
(b) the results at $t=15$. In both (a) and (b), the first column
shows the posterior GPs for category-0 and category-1 along with the
true functions (shown in green color). The second column shows the
estimated distributions of $f_{0}^{*}$ and $f_{1}^{*}$ with their
true values (shown as green dots). \label{fig:An-Illustration-of-Bandit-BO}}
\end{figure}

\begin{algorithm}
\KwIn{$C$: \# of arms, $B$: batch size}

\Begin{

\For{$t=1,2,...$}{

\ForEach{$c\in\{1,...,C\}$}{

fit $\text{GP}_{c}$ (\emph{i.e.} $p(f_{c}(x)\mid{\cal D}_{t}^{c})$)
using ${\cal D}_{t}^{c}$\;

}

\For{$b=1$ \KwTo$B$}{

\ForEach{$c\in\{1,...,C\}$}{

draw $\tilde{f}_{c}(x)\sim p(f_{c}(x)\mid{\cal D}_{t}^{c})$\;

obtain $\tilde{x}_{c}^{*}=\text{argmax}_{x\in{\cal X}_{c}}\tilde{f}_{c}(x)$\; 

set $\tilde{f}_{c}^{*}=\tilde{f}_{c}(\tilde{x}_{c}^{*})$\;

}

choose an arm $c_{t+b}=\arg\max_{c}\tilde{f}_{c}^{*}$\;

suggest a point $x_{t+b}=\tilde{x}_{c_{t+b}}^{*}$\;

evaluate $y_{t+b}=f_{c_{t+b}}(x_{t+b})+\epsilon_{t+b}$\;

}

${\cal D}_{t+B}={\cal D}_{t}\cup\{c_{t+b},x_{t+b},y_{t+b}\}_{b=1}^{B}$\;

}

}

\caption{\label{alg:Bandit-BO-algorithm-NEW}The proposed \textbf{Bandit-BO}
algorithm.}

\end{algorithm}

We have several choices for MAB algorithms such as UCB1, Exp3, $\epsilon$-greedy,
and Thompson sampling \cite{auer2002finite,russo2014learning} and
similarly multiple choices for BO algorithms differing mainly in acquisition
functions e.g. EI, GP-UCB, TS, entropy search etc \cite{hernandez2014predictive}.
We prefer to use GP-UCB\footnote{Although we focused only on TS for BO, but GP-UCB is also feasible
both practically and theoretically. The key difference is that we
get probabilistic regret bounds holding with high probability.} or TS for BO as these algorithms can be analyzed to provide theoretical
upper bounds on regret. For MAB algorithm, we have decided to use
TS for multiple reasons: (1) the GP directly offers a posterior distribution
on the arm means $f_{c}^{*}$ and therefore using TS is feasible;
(2) TS is flexible to modifications and amenable to theoretical analyses;
and (3) TS is shown to achieve competitive performance in practice
due to using a complete distribution for exploration/exploitation
trade-off \cite{kandasamy2018parallelised}.

Using TS for both BO and MAB, our method works as follows. At each
iteration $t$, for each arm $c$ we first model $f_{c}(x)$ via a
GP using existing observations ${\cal D}_{t}^{c}=\{c_{i},x_{i},y_{i}=f_{c_{i}}(x_{i})+\epsilon_{i}\mid c_{i}=c\}_{i=1}^{t}$.
Collectively, we denote ${\cal D}_{t}=\cup_{c=1}^{C}{\cal D}_{t}^{c}$.
We then randomly draw a function from the posterior GP distribution
of each arm as $\tilde{f}_{c}(x)\sim p(f_{c}(x)\mid{\cal D}_{t}^{c})$.
Next we select an arm as $c_{t+1}=\argmax{c\in\{1,...,C\}}\max_{x\in\mathcal{X}_{c}}\tilde{f}_{c}(x)$.
Given the selected arm $c_{t+1}$ (or category), we need to recommend
a value for $x_{t+1}$. Usually in BO, $x_{t+1}$ is recommended via
an acquisition function. Since we are using TS as the acquisition
function in BO, we recommend $x_{t+1}=\text{argmax}_{x\in\mathcal{X}_{c_{t+1}}}\tilde{f}_{c_{t+1}}(x)$.
Finally, we observe the function value as $y_{t+1}=f_{c_{t+1}}(x_{t+1})+\epsilon_{t+1}$
and update the observation set as ${\cal D}_{t+1}^{c_{t+1}}={\cal D}_{t}^{c_{t+1}}\cup\{c_{t+1},x_{t+1},y_{t+1}\}$.
We call our method \textbf{Bandit-BO}.

\subsection{Batch Setting\label{subsec:Batch-Bandit-BO}}

We extend our sequential method to \textit{batch setting} to recommend
a set of $B$ samples at each round. We note that the number of total
function evaluations $T$ = the number of rounds $N$ $\times$ the
batch size $B$. Our batch algorithm is similar to the sequential
one except that at each round, we recommend $B$ samples of $c$ and
$x$, each sample obtained using an independent Thompson sample from\emph{
$p(f_{c}(x)\mid{\cal D}_{t}^{c})$.} Using these recommendations $\{(c_{t+b},x_{t+b})\}_{b=1}^{B}$,
we evaluate the functions as $y_{t+b}=f_{c_{t+b}}(x_{t+b})+\epsilon_{t+b}$
and update the observation set as ${\cal D}_{t+B}={\cal D}_{t}\cup\{c_{t+b},x_{t+b},y_{t+b}\}_{b=1}^{B}$.

Our \textbf{Bandit-BO} in both sequential and batch settings is summarized
in Algorithm \ref{alg:Bandit-BO-algorithm-NEW}.

%% file: framework_part2.tex
\subsection{Convergence Analysis}

We first present the convergence analysis for the sequential setting
and then extend it to the batch setting. Our analysis is developed
on the previous theoretical results of \cite{russo2014learning,desautels2014parallelizing,kandasamy2018parallelised}.

\subsubsection{Sequential Setting\label{subsec:Convergence-Sequential-Bandit-BO}}

TS has been analyzed earlier mostly in the context of MAB with finite
arms. An exception is \cite{russo2014learning}, which extended the
analysis to GP bandits to infinitely many dependent arms. BO using
GP models is a related problem where one has to decide the best point
among an uncountably infinite set of points specified by $\mathcal{X}_{c}$.
In our analysis, we extend the results of \cite{russo2014learning}
to advance the BO in a joint space of categorical and continuous variables.
Our analysis provides the convergence guarantee using \textit{Bayesian
regret}, which has been used as regret measure by several earlier
works \cite{agrawal2013further,Bubeck_Liu_2013prior}.

Following \cite{russo2014learning}, the Bayesian regret of our proposed
\textbf{Bandit-BO} after $T$ iterations is
\begin{equation}
\text{BayesRegret}(T)=\mathbb{E}\sum_{t=1}^{T}[f_{c^{*}}(x^{*})-f_{c_{t}}(x_{t})],
\end{equation}
where the expectation is w.r.t. a distribution over all possible functions
$f_{c}$ in our hypothesis space and any randomness in the algorithm,
particularly the random sampling of TS. Inserting and deleting $f_{c_{t}}(x_{c_{t}}^{*})$
in the above expression we can write the $\text{BayesRegret}(T)$
as

\begin{align}
\text{BayesRegret}(T) & =\mathbb{E}\sum_{t=1}^{T}[f_{c^{*}}(x^{*})-f_{c_{t}}(x_{c_{t}}^{*})]+\label{eq:Regret_decomposition}\\
 & \mathbb{E}\sum_{t=1}^{T}[f_{c_{t}}(x_{c_{t}}^{*})-f_{c_{t}}(x_{t})]=R_{T}^{\text{MAB}}+R_{T}^{\text{BO}}\nonumber 
\end{align}
where we have defined $R_{T}^{\text{MAB}}\triangleq\mathbb{E}\sum_{t=1}^{T}[f_{c^{*}}(x^{*})-f_{c_{t}}(x_{c_{t}}^{*})]$
and $R_{T}^{\text{BO}}\triangleq\mathbb{E}\sum_{t=1}^{T}[f_{c_{t}}(x_{c_{t}}^{*})-f_{c_{t}}(x_{t})]$.
The reason for using the terminology of $R_{T}^{\text{MAB}}$ and
$R_{T}^{\text{BO}}$ is as follows: since $f_{c}(x_{c}^{*})$ is the
mean of arm $c$ rewards, $f_{c^{*}}(x^{*})-f_{c_{t}}(x_{c_{t}}^{*})$
denotes the regret due to choosing a sub-optimal arm at iteration
$t$; and similarly, given the choice of $c_{t}$ at iteration $t$,
$f_{c_{t}}(x_{c_{t}}^{*})-f_{c_{t}}(x_{t})$ denotes the regret of
BO choosing a sub-optimal continuous point.

To provide an upper bound on $\text{BayesRegret}(T)$ in Eq. (\ref{eq:Regret_decomposition}),
we prove Lemma \ref{lem:MAB_regret_Seq} and \ref{lem:BO_regret_seq},
which provide upper bounds on $R_{T}^{\text{MAB}}$ and $R_{T}^{\text{BO}}$
respectively. Before we proceed, we need to state the following two
assumptions. The first assumption is required to prove regret bounds
for BO in a continuous search domain \cite{srinivas2012information,kandasamy2018parallelised}.
The second assumption is required to prove Lemma \ref{lem:MAB_regret_Seq}
and \ref{lem:BO_regret_seq}.

\noindent\textbf{Assumption 1.} Let $\{f_{c}(x)\}_{c=1}^{C}$ be
a set of functions such that $x\in\mathcal{X}_{c}\subset\mathbb{R}^{d}$.
Further, for all $c$, let $f_{c}(x)\sim\text{GP}_{c}(0,k_{c})$ with
a covariance function $k_{c}$ such that for any sample path of $\textrm{GP}_{c}$,
there exist constants $r$ and $s$ such that its partial derivatives
satisfy the following condition
\[
\forall L>0,\forall i\in\{1,\ldots,d\}\ \mathbb{P}(|\partial f/\partial x_{i}|<L)\geq1-dre^{-L^{2}/s^{2}}
\]

\noindent\textbf{Assumption 2.} $\forall c\in\{1,...,C\}$, let $f_{c}(x)\sim\text{GP}_{c}(0,k_{c})$
with a covariance function $k_{c}$ such that the maximum information
gain $\gamma_{T_{c}}$ about $f_{c}$ due to any $T_{c}$ noisy observations
is strictly sub-linear in $T_{c}$. Therefore, there exists an $\alpha$
such that $\gamma_{T_{c}}\sim\mathcal{O}(T_{c}^{\alpha})$ where $0\leq\alpha<1$.

As stated in \cite{srinivas2012information}, Assumption 1 satisfies
for any covariance function that is four times differentiable. Assumption
2 satisfies for common covariance functions e.g. squared-exponential,
Mat\'{e}rn etc.
\begin{lem}[Upper bound on $R_{T}^{\text{MAB}}$]
\label{lem:MAB_regret_Seq}Let $f_{c}(x)\sim\text{GP}_{c}(0,k_{c}),c=1,...,C$
and $\mathcal{D}_{t}=\{(c_{i},x_{i},y_{i})\}_{i\leq t}$ be the noisy
function observation set suggested by our method in the \textbf{sequential
setting} under the observation model $y_{i}=f_{c_{i}}(x_{i})+\epsilon_{i}$,
where $\epsilon_{i}\sim\mathcal{N}(0,\sigma_{\epsilon}^{2})$, then
under Assumptions 1 and 2, we have $R_{T}^{\text{MAB}}\leq\mathcal{O}\left(\sqrt{CT^{\alpha+1}\log T}\right)$.
\end{lem}

\begin{proof}
The proof is provided in supplementary material.
\end{proof}
\begin{lem}[Upper bound on $R_{T}^{\text{BO}}$]
\label{lem:BO_regret_seq}Let $f_{c}(x)\sim\text{GP}_{c}(0,k_{c}),c=1,...,C$
and $\mathcal{D}_{t}=\{(c_{i},x_{i},y_{i})\}_{i\leq t}$ be the noisy
function observation set suggested by our method in the \textbf{sequential
setting} under the observation model $y_{i}=f_{c_{i}}(x_{i})+\epsilon_{i}$,
where $\epsilon_{i}\sim\mathcal{N}(0,\sigma_{\epsilon}^{2})$, then
under Assumptions 1 and 2, we have $R_{T}^{\text{BO}}\leq\mathcal{O}\left(\sqrt{CT^{\alpha+1}\log T}\right)$.
\end{lem}

\begin{proof}
The proof is provided in supplementary material.
\end{proof}
Finally, the \emph{overall Bayesian regret} for our \textbf{Bandit-BO}
is stated in Theorem \ref{thm:Overall-Bayesian-regret-Seq}. 
\begin{thm}[$\text{BayesRegret}(T)$ for Sequential Setting]
\label{thm:Overall-Bayesian-regret-Seq}Under Assumptions 1 and 2,
the Bayesian regret for our method after $T$ iterations in the \textbf{sequential
setting} is bounded as 
\begin{equation}
\text{BayesRegret}\left(T\right)\leq\mathcal{O}\left(\sqrt{CT^{\alpha+1}\text{log}T}\right)\label{eq:Overall-Bayesian-regret-Seq}
\end{equation}
\end{thm}

\begin{proof}
The proof follows by combining the results of Lemma \ref{lem:MAB_regret_Seq}
and \ref{lem:BO_regret_seq} with Eq. (\ref{eq:Regret_decomposition}).
We note that the regret grows only sub-linearly in both $T$ and $C$.
\end{proof}

\paragraph*{Efficiency of using MAB:}

The Bayesian regret of our proposed \textbf{Bandit-BO} algorithm is
only sub-linear in $T$ as seen in Eq. (\ref{eq:Overall-Bayesian-regret-Seq}).
We compare it with \emph{two extreme settings}: the \emph{first extreme}
where the optimal arm (or category) is known, and the \emph{second
extreme} where each arm gets equal allocation e.g. visiting each arm
in a round-robin fashion or sampling an arm uniformly randomly. For
the first extreme, an oracle who knows the optimal arm $c^{*}$ will
allocate all $T$ iterations to the arm $c^{*}$ and therefore will
have $R_{T}^{\text{MAB}}=0$ and $R_{T}^{\text{BO}}\leq\mathcal{O}(\sqrt{T^{\alpha+1}\text{log}T})$.
On the other hand, the second extreme being a na\"{i}ve algorithm
allocating equal budget to each arm will incur $R_{T}^{\text{MAB}}=\sum_{c\neq c^{*}}\Delta_{c}\frac{T}{C}=\mathcal{O}(T$)
(where $\Delta_{c}=f_{c^{*}}^{*}-f_{c}^{*}$ denoting the sub-optimality
of each arm) and $R_{T}^{\text{BO}}\leq\mathcal{O}(\sqrt{T^{\alpha+1}\text{log}T})$.
This results in a total regret that grows linearly in $T$. Thus,
our method with regret upper bound $\left(\sqrt{CT^{\alpha+1}\text{log}T}\right)$
is a significantly better algorithm than equal-budget allocation algorithms
and comparable to the Oracle with just an extra sub-linear factor
$\sqrt{C}$.

\subsubsection{Batch Setting\label{subsec:ConvergenceBatch-Bandit-BO}}

The main difference between the analysis of a sequential and batch
algorithm arises from the way function values are observed. Unlike
a sequential setting where we observe the function value immediately
after recommending a sample, a batch setting with batch size $B$
gets to observe the functions values only after recommending $B$
samples. Due to the late feedback on the function knowledge, $\sigma_{t_{c}}^{c}(x)$,
the predictive variance of $f_{c}(x)$ in the batch setting, is higher
than that in the sequential setting at any iteration $t_{c}$ prior
to which there was a recommendation without function value observation.
Desautels et al. \cite{desautels2014parallelizing} showed that this
gap in the function knowledge (or increased uncertainty) due to the
batch setting can be bounded for any $x\in\mathcal{X}$ as
\begin{equation}
\left(\sigma_{t_{c}}^{c}(x)\right)_{\text{batch}}\leq\left(\sigma_{t_{c}}^{c}(x)\right)_{\text{seq}}\psi_{B},\label{eq:Desautel_inequality}
\end{equation}
where $\psi_{B}$ is a sub-linear term in $B$ related to the maximum
information gain potentially brought by any $B$ samples. Eq. (\ref{eq:Desautel_inequality})
gives us $\sum_{t_{c}=1}^{T_{c}}\left(\sigma_{t_{c}}^{c}(x)\right)_{\text{batch}}^{2}\leq\psi_{B}^{2}\sum_{t_{c}=1}^{T_{c}}\left(\sigma_{t_{c}}^{c}(x)\right)_{\text{seq}}^{2}$.
Since $\sum_{t_{c}=1}^{T_{c}}\left(\sigma_{t_{c}}^{c}(x)\right)_{\text{seq}}^{2}\leq\gamma_{T_{C}}$
\cite{srinivas2012information}, we have $\sum_{t_{c}=1}^{T_{c}}\left(\sigma_{c}^{t_{c}}(x)\right)_{\text{batch}}^{2}\leq\psi_{B}^{2}\gamma_{T_{C}}$.
For deriving regret bounds for the batch setting, to bound $\sum_{t_{c}=1}^{T_{c}}\left(\sigma_{c}^{t_{c}}(x)\right)_{\text{batch}}^{2}$,
we can use $\psi_{B}^{2}\gamma_{T_{C}}$ instead of $(\gamma_{T_{c}})_{\textrm{batch}}$.
Since $\psi_{B}$ is independent of $T_{c}$, we can extend Theorem
\ref{thm:Overall-Bayesian-regret-Seq} of sequential setting to the
batch setting as stated in the following Theorem.
\begin{thm}[$\text{BayesRegret}(T)$ for Batch Setting]
Under Assumptions 1 and 2, the Bayesian regret for our method after
$T$ iterations in the \textbf{batch setting} is given as
\[
\text{BayesRegret}\left(T\right)\leq\mathcal{O}\left(\psi_{B}\sqrt{CT^{\alpha+1}\text{log}T}\right)
\]
\end{thm}

\begin{proof}
The proof is provided in supplementary material.
\end{proof}

%% file: discussion.tex
Our algorithm is capable of handling cases where the search space
for continuous variables for each category is either identical (\emph{category-independent}
continuous search spaces) or different (\emph{category-specific} continuous
search spaces). We used an independent GP to model the continuous
function $f_{c}(x)$ for each category $c$. When the search spaces
for continuous variables are identical for all categories, it may
be useful to incorporate any correlations across categories. In our
algorithm, it is possible to incorporate such correlations through
a multi-task Gaussian process (MTGP) \cite{Bonilla_etal_08multi}.
MTGP allows to use a covariance function of the form $k(c,x,c',x')$,
which can be factorized as $k(c,c')\times k(x,x')$, and through $k(c,c')$
we can incorporate the correlation across the categories. A possible
example of $k(c,c')$ is Hamming kernel that was used by \cite{wang2016bayesian}.
However, in this paper, since our focus is to provide a general algorithm
that is applicable to both category-specific and category-independent
continuous search spaces, we have ignored the correlation aspect.
This is because correlation across any two categories having different
continuous search spaces does not make sense.

%% file: experiment.tex
We conduct experiments to show the performance of our proposed \textbf{Bandit-BO}
for both synthetic and real-world applications in sequential and batch
settings.

We compare our method with four state-of-the-art baselines that use
different ways to deal with categorical variables: \textbf{One-hot-Encoding}
\cite{golovin2017google}, \textbf{Merchan-Lobato} \cite{garrido2018dealing},
\textbf{SMAC} \cite{hutter2011sequential}, and \textbf{TPE} \cite{bergstra2011algorithms}.
For One-hot-Encoding and Merchan-Lobato methods, batch recommendation
is made using Thompson sampling. For SMAC, we form the batch using
``hallucinated'' observations similar to \cite{desautels2014parallelizing}.
For TPE, we form the batch using the likelihood based sampling as
described in \cite{bergstra2011algorithms}. In our experiments, we
randomly initialize two points for GP fitting for each category, resulting
in $2C$ initial points in total. The initialization points are kept
identical across all methods for a fair comparison. We repeat each
method 10 times and report the average result along with the standard
error.

\subsection{Synthetic Applications}

The first experiment illustrates how our algorithm performs with different
batch sizes and numbers of categories.

\subsubsection{Synthetic function}

Our synthetic function is created by \emph{modifying} Ackley-5d function
in 5 continuous variables by an extra categorical variable. We shift
the function for each category by a value $c$ as $f([c,x])=-20\exp(-0.2\sqrt{\frac{1}{5}{\displaystyle \sum_{i=1}^{5}z_{i}^{2}}})-\exp(\frac{1}{5}{\displaystyle \sum_{i=1}^{5}\cos(2\pi z_{i}))+20+\exp(1)}+c$,
where $z_{i}=x_{i}+c$. In the supplementary material, we provide
additional experiments for \emph{two more synthetic functions}.

\subsubsection{Study of varying batch sizes\label{subsec:Study-of-batchsize}}

Figure \ref{fig:Optimization-results-synthetic}(a) shows the optimization
results for different batch sizes while fixing the number of categories
to $6$. Our method \textbf{Bandit-BO} is significantly better compared
to the other methods. It shows consistent improvements over $T=120$
iterations. One-hot-Encoding, Merchan-Lobato, and SMAC methods do
not perform well. TPE is the second-best method; however its best
found function value is significantly lower than that of \textbf{Bandit-BO}.

\subsubsection{Study of varying numbers of categories\label{subsec:Study-of-category}}

In Figure \ref{fig:Optimization-results-synthetic}(b), we show the
optimization results for different \textit{large} numbers of categories
while fixing the batch size to 5. We can see that \textbf{Bandit-BO}
clearly outperforms the other methods even with a very large number
of categories (e.g. $C=100$). TPE is still a second-best method.
The performance of Merchan-Lobato becomes worse as $C$ is increased.
Interestingly, One-hot-Encoding shows an improvement as $C$ is increased
up to 100, where it is better than SMAC.

\begin{figure}
\begin{centering}
\subfloat[]{\begin{centering}
\begin{minipage}[t]{0.49\columnwidth}%
\begin{center}
\includegraphics[scale=0.22]{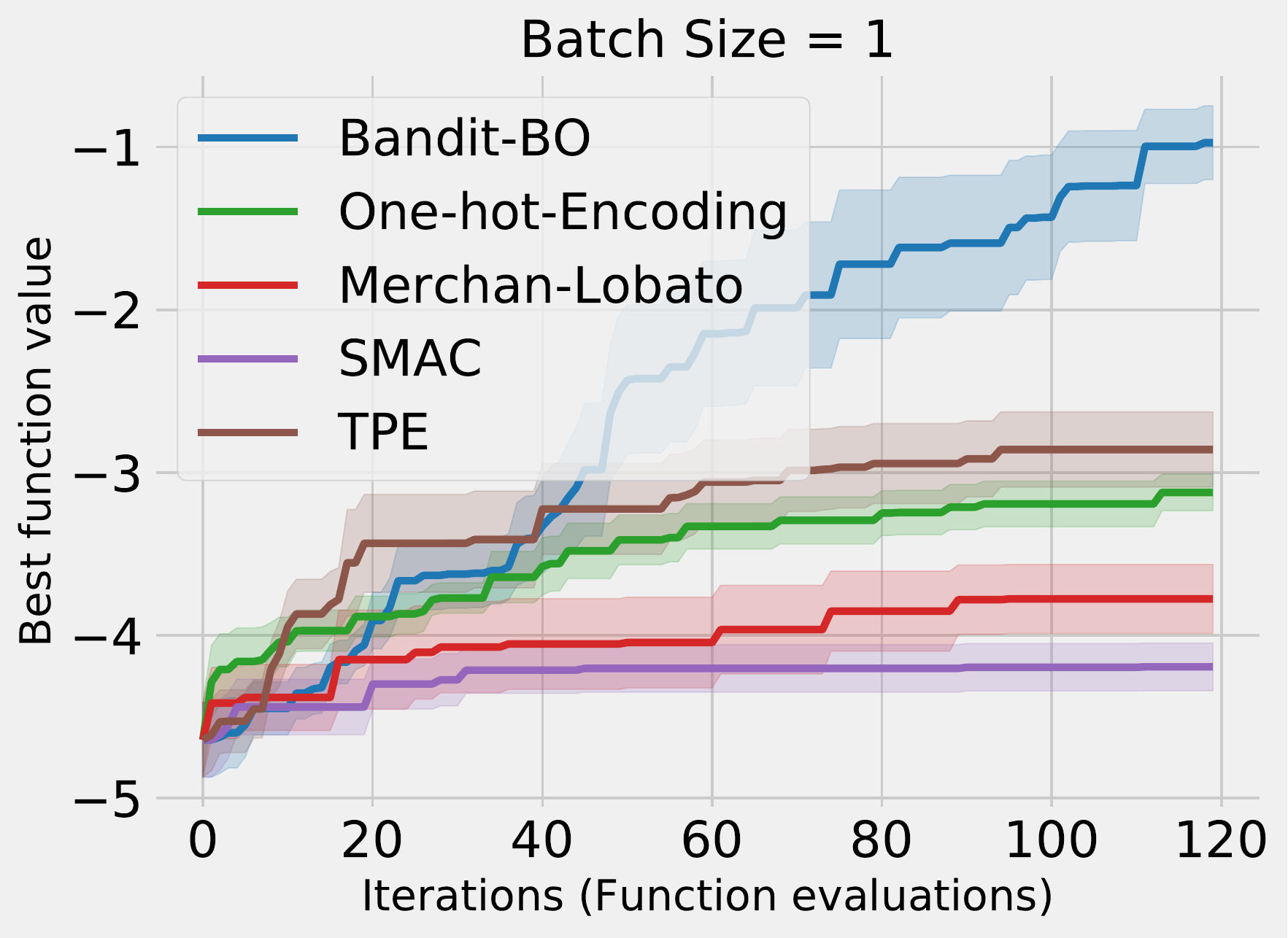}
\par\end{center}
\begin{center}
\includegraphics[scale=0.22]{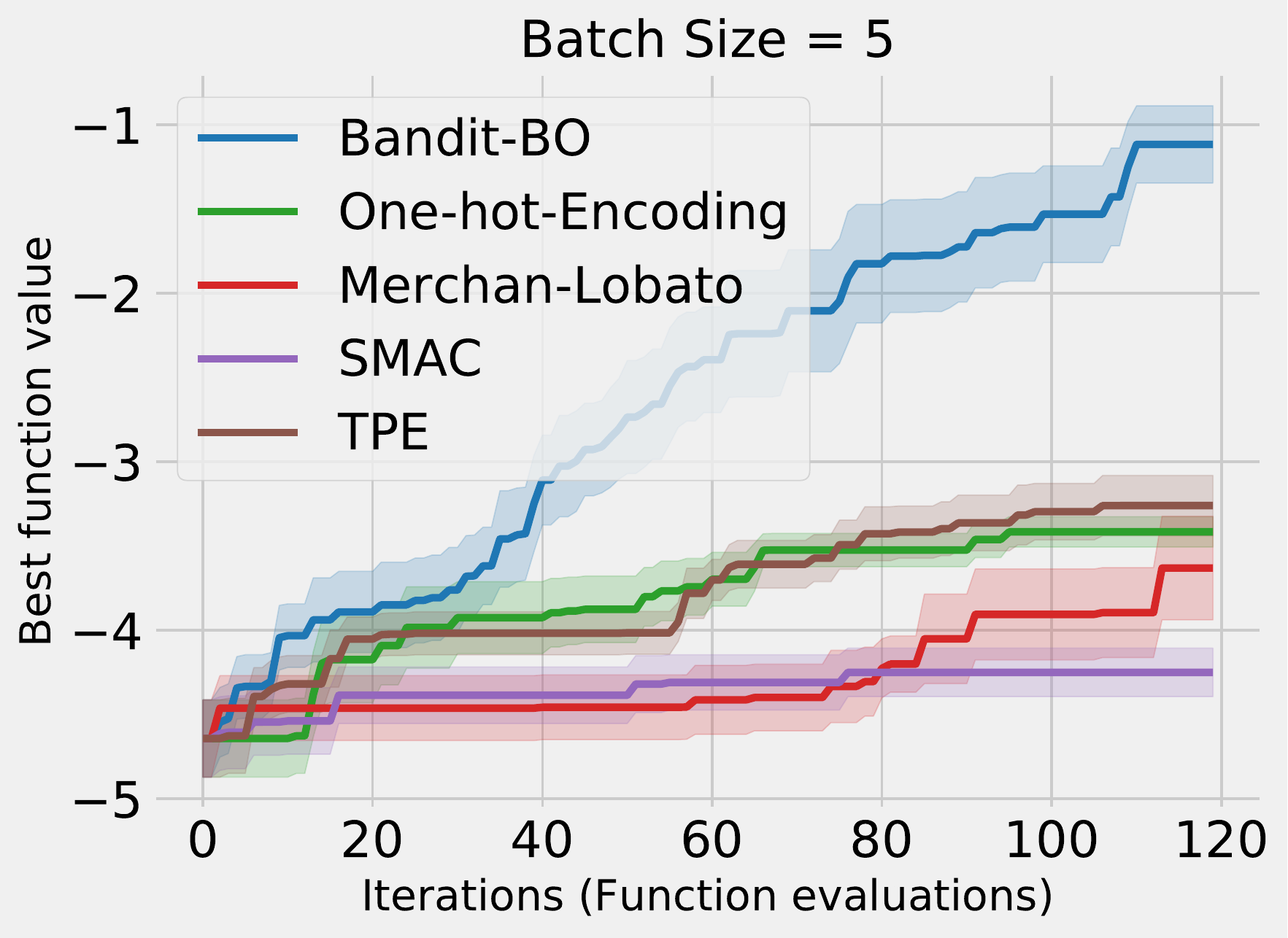}
\par\end{center}
\begin{center}
\includegraphics[scale=0.22]{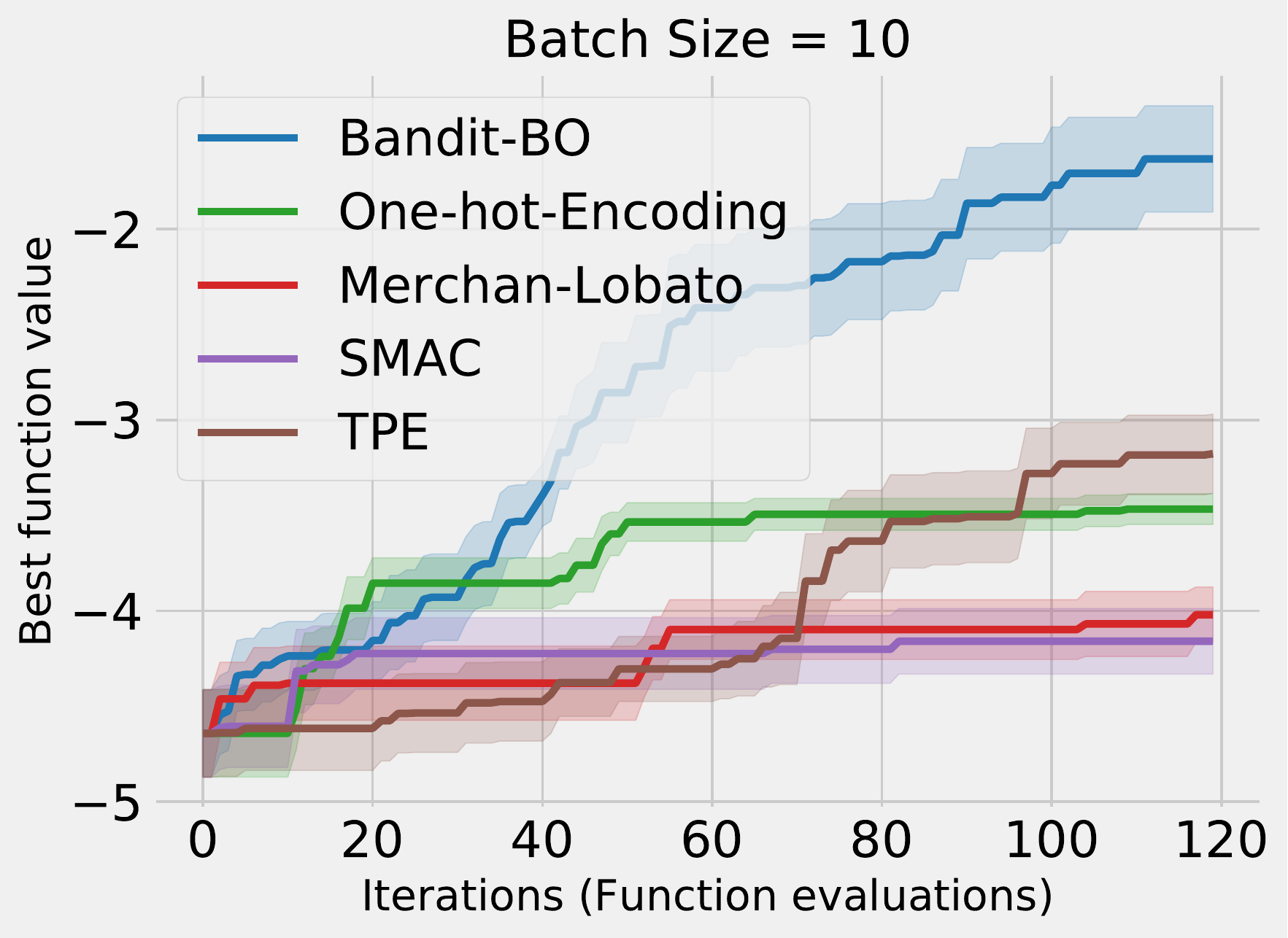}
\par\end{center}%
\end{minipage}
\par\end{centering}
}\subfloat[]{\begin{centering}
\begin{minipage}[t]{0.49\columnwidth}%
\begin{center}
\includegraphics[scale=0.22]{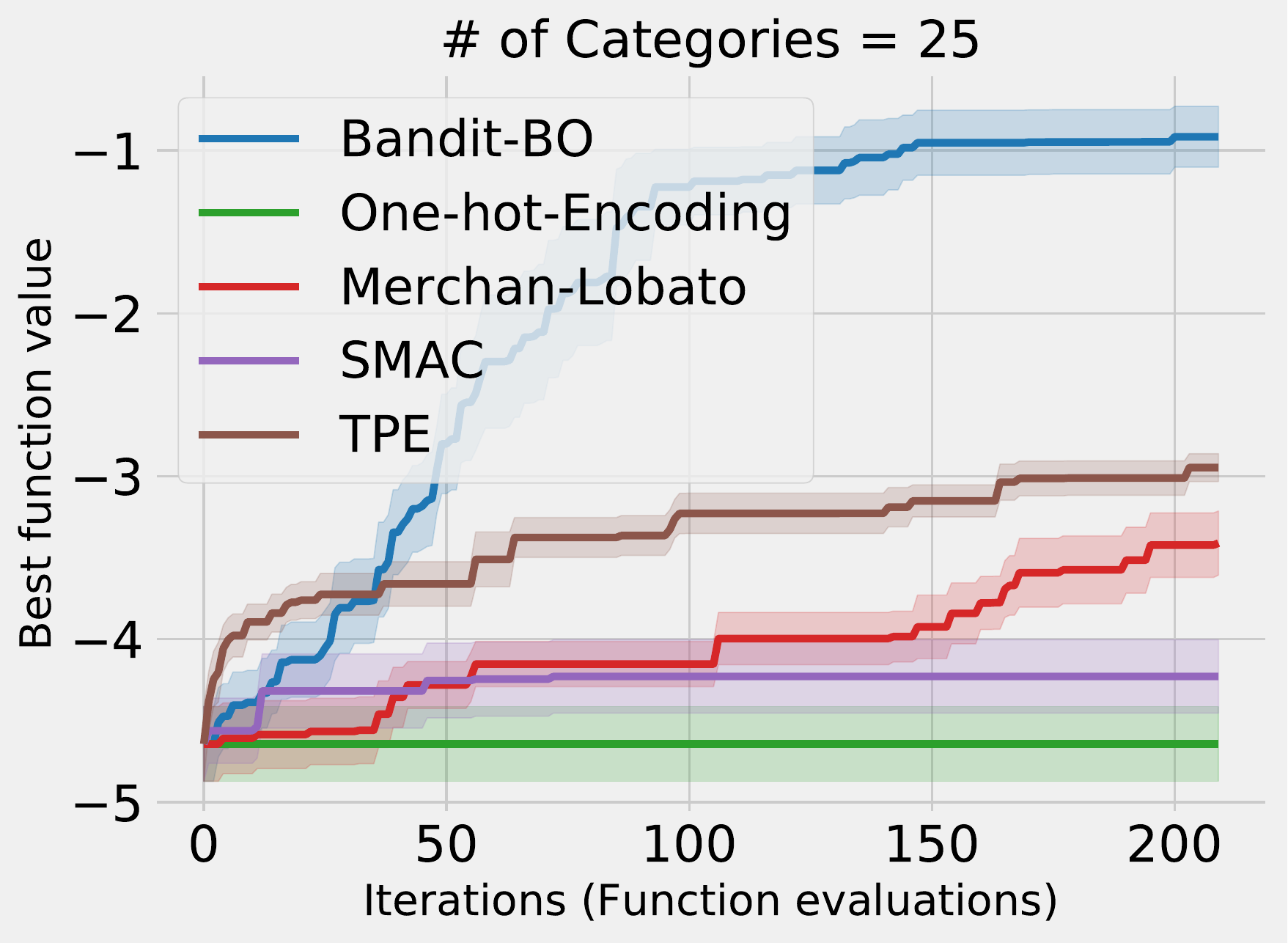}
\par\end{center}
\begin{center}
\includegraphics[scale=0.22]{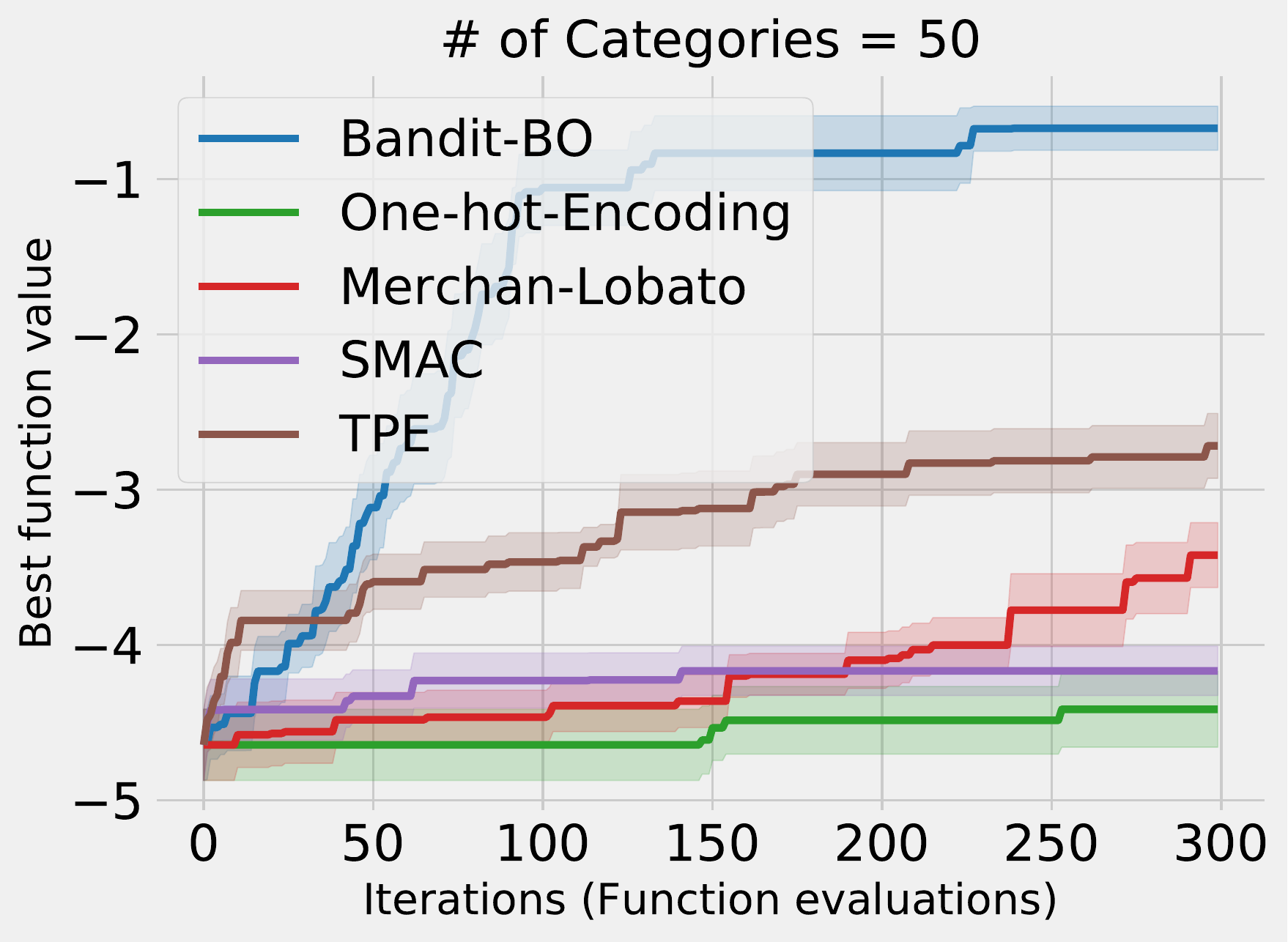}
\par\end{center}
\begin{center}
\includegraphics[scale=0.22]{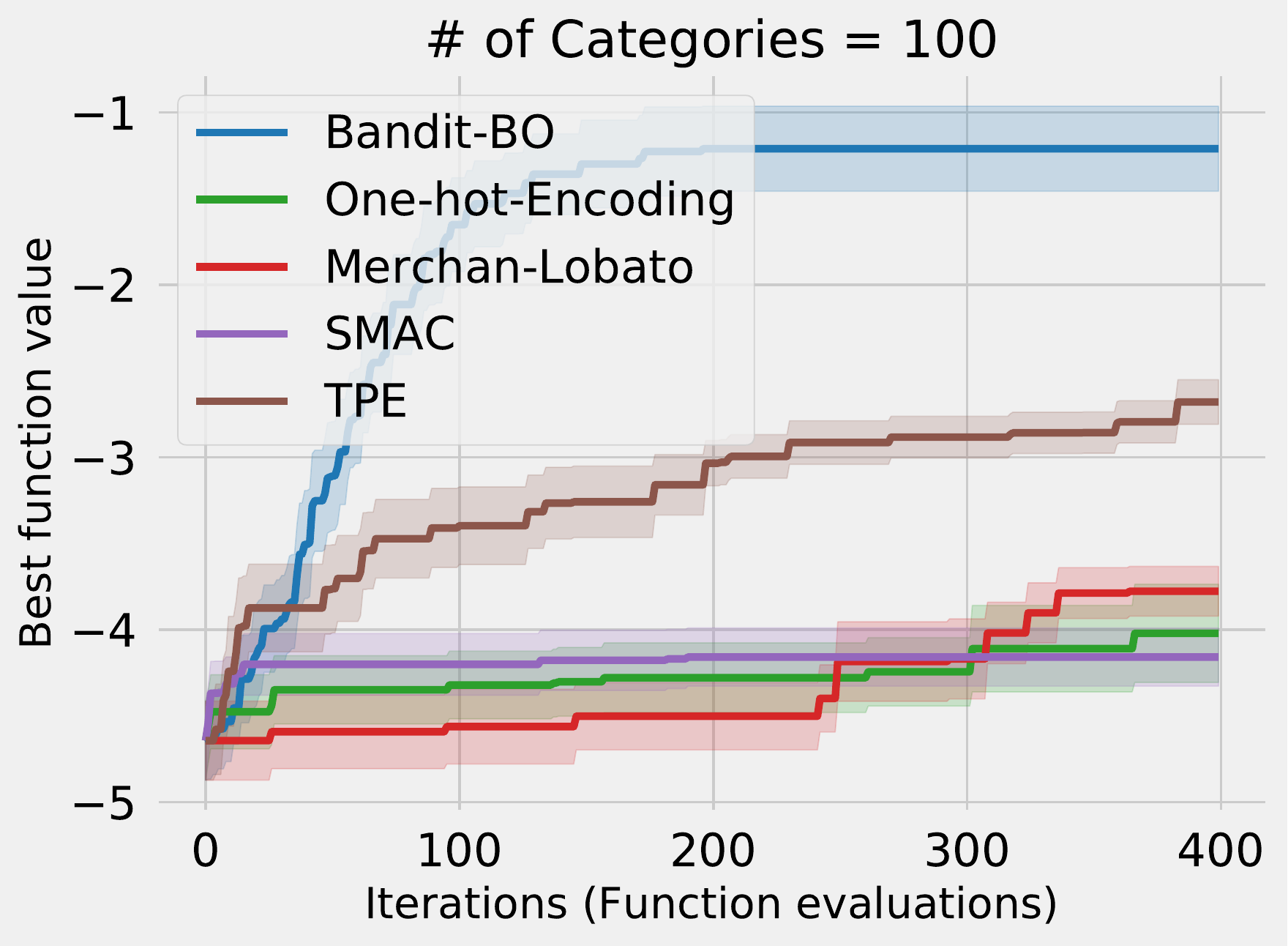}
\par\end{center}%
\end{minipage}
\par\end{centering}
}
\par\end{centering}
\caption{\label{fig:Optimization-results-synthetic}Optimization results for
the \emph{modified} Ackley-5d function (a) for \textit{different batch
sizes}: $B=1$ (sequential), $B=5$, and $B=10$ (the number of categories
is fixed to 6) and (b) for \textit{different numbers of categories}:
$C=25$, $C=50$, and $C=100$ (the batch size is fixed to 5).}

\end{figure}

\subsection{Real-world Applications}

The second experiment shows the efficiency of our method in two real-world
machine learning (ML) applications.

\subsubsection{Hyper-parameter tuning for a feed-forward neural network on a regression
task\label{subsec:Hyper-parameter-tuning}}

Our goal is to find the optimal set of hyper-parameters for a feed-forward
neural network on the \textit{protein structure} dataset\footnote{\url{https://archive.ics.uci.edu/ml/datasets/Physicochemical+Properties+of+Protein+Tertiary+Structure}},
which has 27,438 training points, 9,146 testing points, and nine features.
We define the black-box function as a mapping between the model hyper-parameters
and the \textit{mean squared error} (MSE) on a held-out testing set.
We build the network with two hidden layers and train it using Adam
\cite{kingma2014adam} for 100 epochs. We optimize eight hyper-parameters
as shown in Table \ref{tab:Hyper-parameters-NN}. We report the average
MSE with standard errors for each method as shown in Figure \ref{fig:Results-of-parameter-tune-NN}.
We note that to create arms from four categorical variables, we use
the cross product of their values, resulting in 36 arms (choices)
in total.

\begin{table}
\caption{\label{tab:Hyper-parameters-NN}Hyper-parameters for the neural network.}

\begin{centering}
\begin{tabular}{|l|l|c|}
\hline 
\rowcolor{header_color}\textbf{Type} & \textbf{Hyper-parameter} & \textbf{Values}\tabularnewline
\hline 
\hline 
\multirow{4}{*}{Categorical} & Activation/Layer 1 & $\{\text{tanh, relu}\}$\tabularnewline
\cline{2-3} 
 & Activation/Layer 2 & $\{\text{tanh, relu}\}$\tabularnewline
\cline{2-3} 
 & Layer 1 Size & $2^{\{4,6,9\}}$\tabularnewline
\cline{2-3} 
 & Layer 2 Size & $2^{\{4,6,9\}}$\tabularnewline
\hline 
\multirow{4}{*}{Continuous} & Initial Learning Rate & $10^{[-4,-1]}$\tabularnewline
\cline{2-3} 
 & Batch Size & $2^{[3,6]}$\tabularnewline
\cline{2-3} 
 & Dropout/Layer 1 & $[0.0,0.6]$\tabularnewline
\cline{2-3} 
 & Dropout/Layer 2 & $[0.0,0.6]$\tabularnewline
\hline 
\end{tabular}
\par\end{centering}
\end{table}

From the results in Figure \ref{fig:Results-of-parameter-tune-NN},
we can see that our method \textbf{Bandit-BO} performs the best in
both sequential and batch settings. Similar to the results on the
synthetic function, TPE is the second-best method. When the batch
size is increased up to 5, TPE is slightly comparable to our method.
Merchan-Lobato performs well with all batch sizes, where it is much
better than One-hot-Encoding and SMAC. One-hot-Encoding is slightly
comparable to SMAC as the batch size is increased up to 5 (Figure
\ref{fig:Results-of-parameter-tune-NN}(b)).

\begin{figure*}[t]
\begin{centering}
\subfloat[]{\includegraphics[scale=0.25]{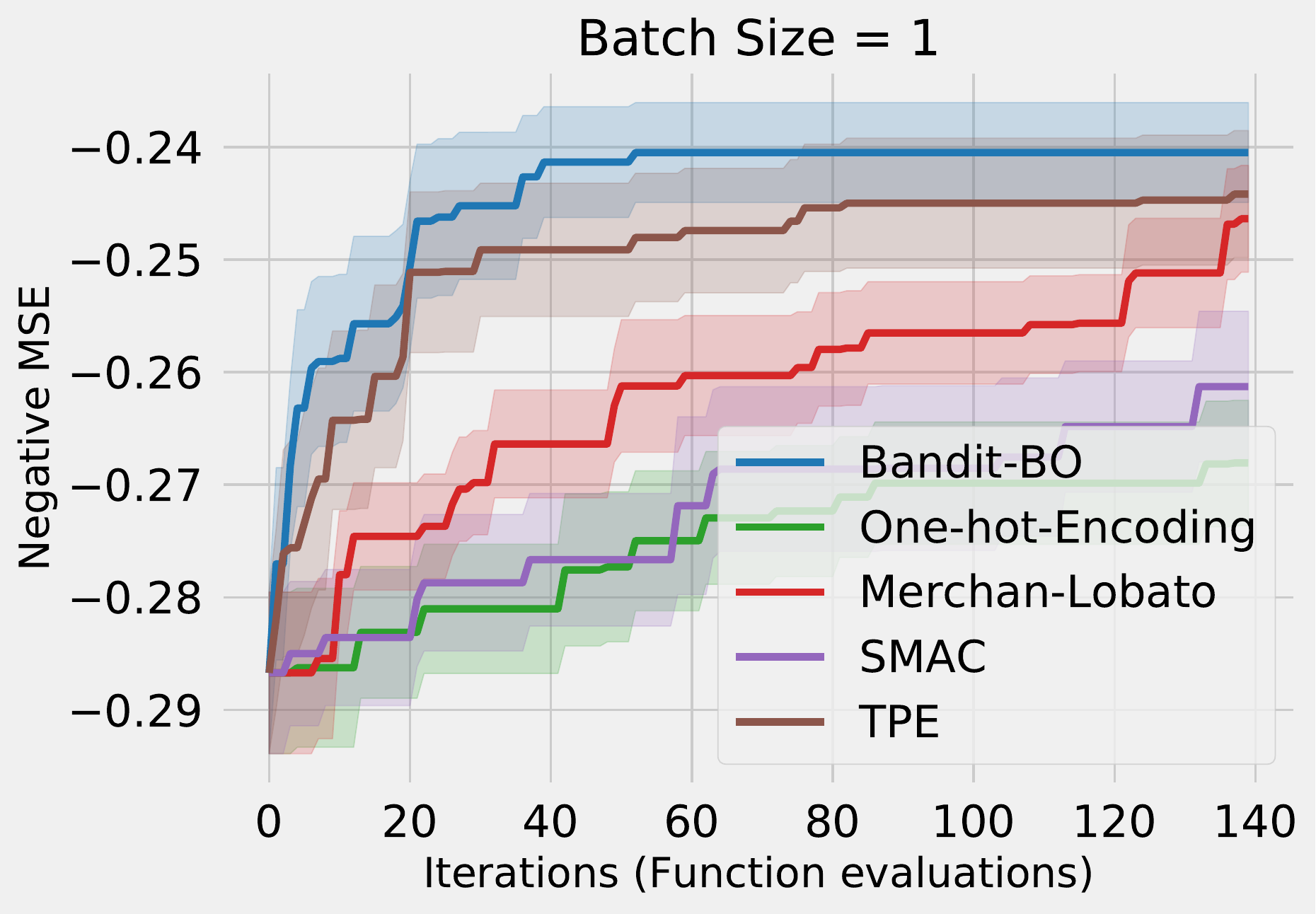}}\hspace{0.4cm}\subfloat[]{\includegraphics[scale=0.25]{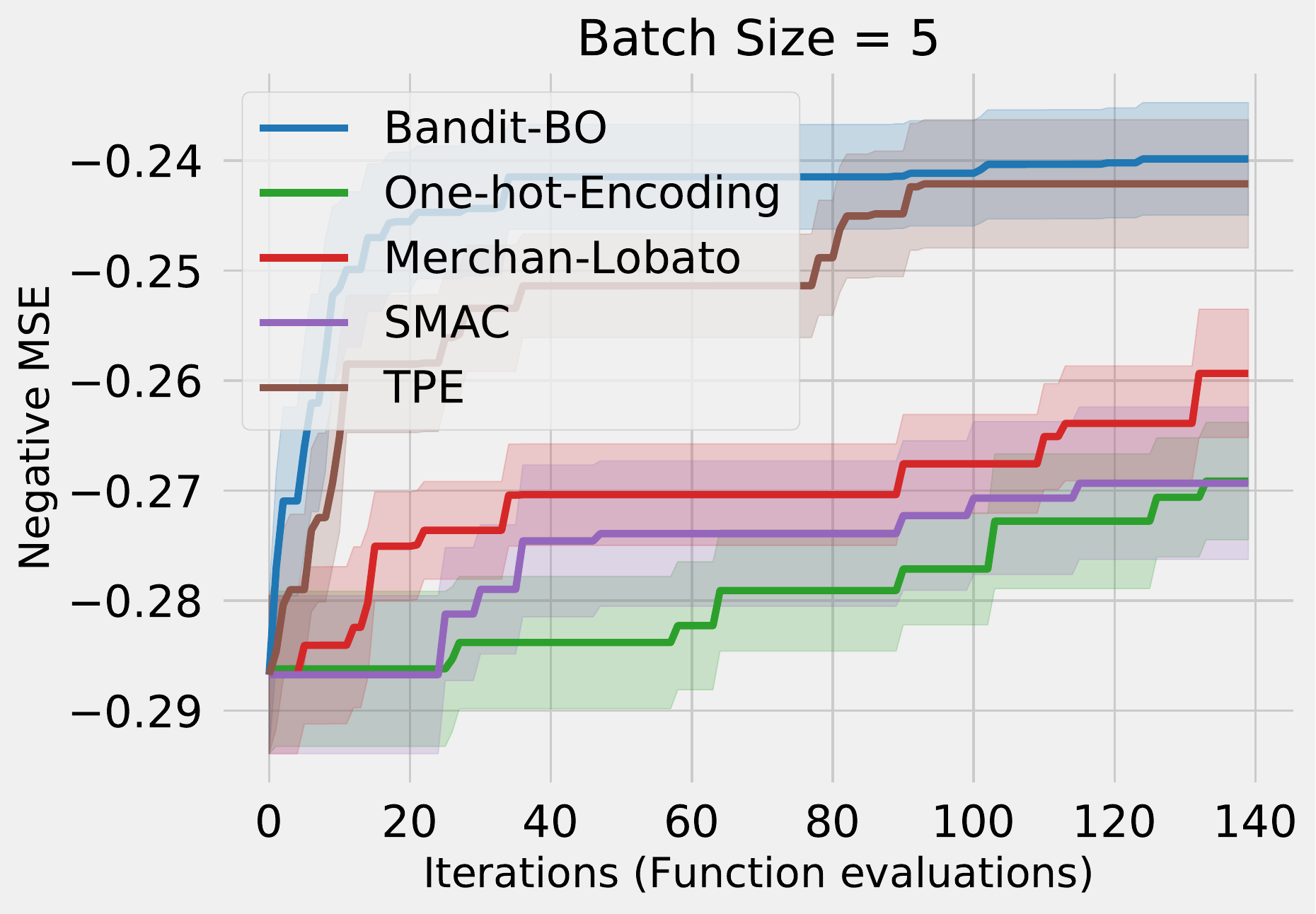}}\hspace{0.4cm}\subfloat[]{\includegraphics[scale=0.25]{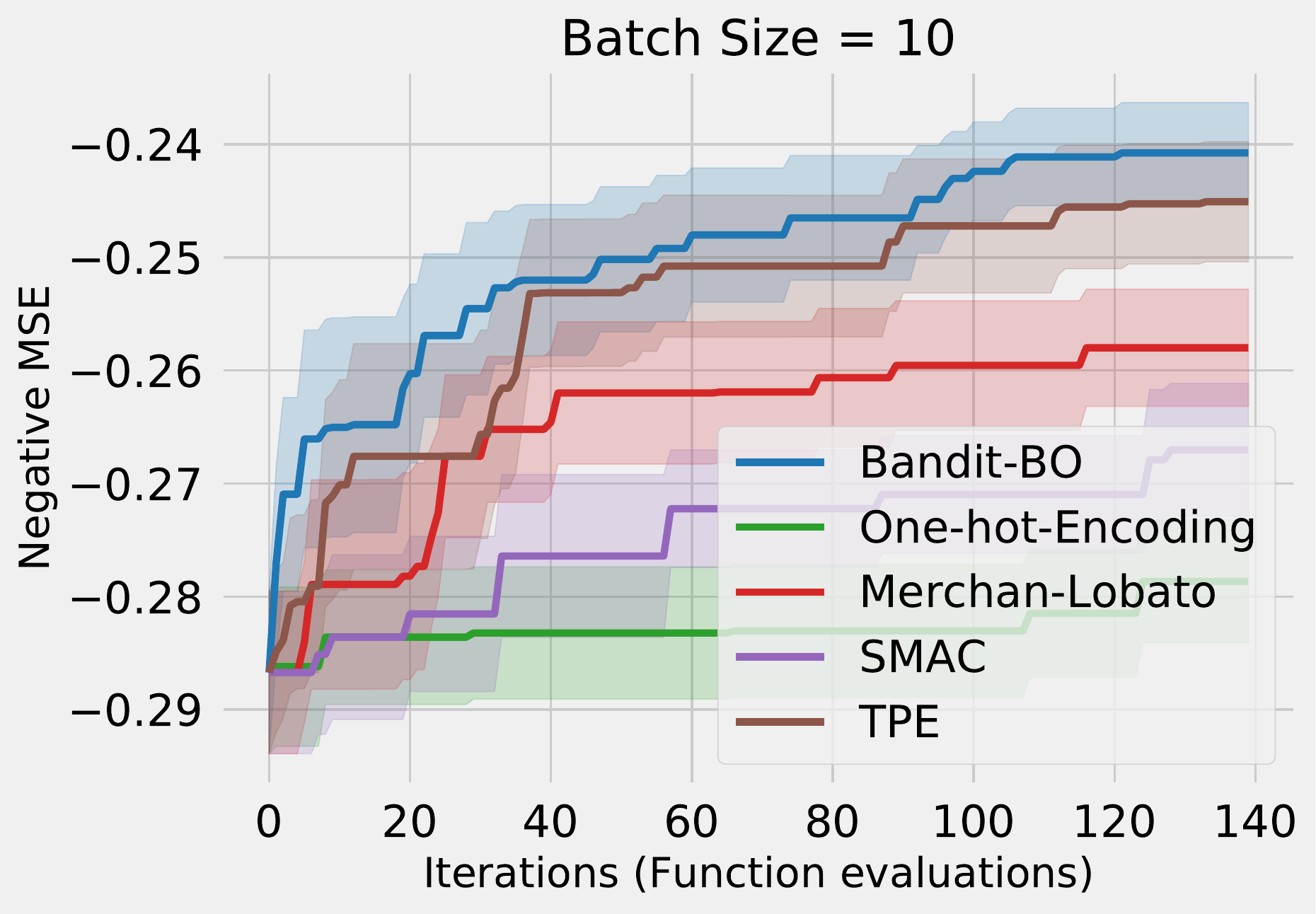}}
\par\end{centering}
\caption{\label{fig:Results-of-parameter-tune-NN}Results of the hyper-parameter
tuning for the neural network -- best value (\textit{negative} MSE)
vs. iteration for batch sizes: (a) $B=1$ (sequential), (b) $B=5$,
and (c) $B=10$.}
\end{figure*}

\begin{table*}
\caption{\label{tab:Individual-accuracy}Characteristics ($|D|$: the number
of samples, $|F|$: the number of features, and $|L|$: the number
of labels) of the first 16 benchmark datasets along with \textit{classification
accuracy} (standard error) of our method \textbf{Bandit-BO} and other
methods. Bold font marks the best performance in a row. The results
for full 30 datasets are reported in the supplementary material.}

\centering{}%
\begin{tabular}{|l|l|r|r|r|>{\raggedleft}p{0.1\textwidth}|>{\raggedleft}p{0.1\textwidth}|>{\raggedleft}p{0.1\textwidth}|>{\raggedleft}p{0.1\textwidth}|}
\hline 
\rowcolor{header_color}\textbf{Dataset} & \textbf{Format} & $|D|$ & $|F|$ & $|L|$ & \textbf{Bandit-BO} & \textbf{Hyperopt-sklearn} & \textbf{Auto-sklearn} & \textbf{TPOT}\tabularnewline
\hline 
\hline 
\textit{wine} & tabular & 178 & 13 & 3 & \textbf{98.33 (0.00)} & 97.78 (0.01) & 97.50 (0.01) & 96.67 (0.01)\tabularnewline
\hline 
\rowcolor{even_color}\textit{breast\_cancer} & tabular & 569 & 30 & 2 & \textbf{97.02 (0.01)} & 95.44 (0.00) & 96.40 (0.00) & 96.84 (0.00)\tabularnewline
\hline 
\textit{analcatdata\_authorship} & text & 841 & 70 & 4 & \textbf{99.76 (0.00)} & 99.47 (0.00) & 99.41 (0.00) & 99.53 (0.00)\tabularnewline
\hline 
\rowcolor{even_color}\textit{diabetes} & tabular & 768 & 8 & 2 & \textbf{77.40 (0.01)} & 73.70 (0.02) & 76.95 (0.01) & 77.01 (0.01)\tabularnewline
\hline 
\textit{electricity} & tabular & 45,312 & 8 & 2 & \textbf{92.29 (0.00)} & 92.21 (0.00) & 90.89 (0.00) & 90.94 (0.00)\tabularnewline
\hline 
\rowcolor{even_color}\textit{wall\_robot\_navigation} & trajectory & 5,456 & 24 & 4 & \textbf{99.73 (0.00)} & \textbf{99.73 (0.00)} & 99.43 (0.00) & 99.46 (0.00)\tabularnewline
\hline 
\textit{vehicle} & tabular & 846 & 18 & 4 & \textbf{81.71 (0.01)} & 78.71 (0.01) & 80.24 (0.01) & 78.12 (0.01)\tabularnewline
\hline 
\rowcolor{even_color}\textit{cardiotocography} & tabular & 2,126 & 35 & 10 & \textbf{100.0 (0.00)} & \textbf{100.0 (0.00)} & 99.98 (0.00) & \textbf{100.0 (0.00)}\tabularnewline
\hline 
\textit{artificial\_characters} & text & 10,218 & 7 & 10 & 90.47 (0.01) & \textbf{90.94 (0.00)} & 82.49 (0.00) & 87.75 (0.01)\tabularnewline
\hline 
\rowcolor{even_color}\textit{monks1} & tabular & 556 & 6 & 2 & \textbf{100.0 (0.00)} & 99.82 (0.00) & 99.73 (0.00) & \textbf{100.0 (0.00)}\tabularnewline
\hline 
\textit{monks2} & tabular & 601 & 6 & 2 & 98.26 (0.01) & 97.69 (0.01) & 97.36 (0.01) & \textbf{99.92 (0.00)}\tabularnewline
\hline 
\rowcolor{even_color}\textit{steel\_plates\_fault} & tabular & 1,941 & 33 & 2 & \textbf{100.0 (0.00)} & \textbf{100.0 (0.00)} & \textbf{100.0 (0.00)} & \textbf{100.0 (0.00)}\tabularnewline
\hline 
\textit{phoneme} & tabular & 5,404 & 5 & 2 & \textbf{90.23 (0.00)} & 90.21 (0.00) & 89.25 (0.00) & 89.58 (0.00)\tabularnewline
\hline 
\rowcolor{even_color}\textit{waveform} & tabular & 5,000 & 40 & 3 & \textbf{86.45 (0.00)} & 86.42 (0.00) & 86.19 (0.00) & 86.28 (0.00)\tabularnewline
\hline 
\textit{balance\_scale} & tabular & 625 & 4 & 3 & \textbf{98.48 (0.01)} & 97.20 (0.01) & 89.04 (0.01) & 92.32 (0.01)\tabularnewline
\hline 
\rowcolor{even_color}\textit{digits} & image & 1,797 & 64 & 10 & 98.25 (0.00) & \textbf{98.67 (0.00)} & 98.08 (0.00) & 97.86 (0.00)\tabularnewline
\hline 
\end{tabular}
\end{table*}

\subsubsection{\emph{Automated Machine Learning}\textit{\emph{: Automatic selection
of the best ML model along with its optimal hyper-parameters\label{subsec:Automated-Machine-Learning}}}}

Given a dataset and several candidate ML models, e.g. decision tree,
random forest, logistic regression, support vector machine, etc, our
goal is to determine which model along with its optimized hyper-parameters
produces the highest \textit{accuracy} on the dataset. We formulate
this task as a black-box function $f([c,x_{c}])$ optimization, where
$c$ indexes a ``ML model'' and $x_{c}$ is a set of hyper-parameters
specified for that model (the detail of ML models and their hyper-parameters
are provided in the supplementary material). We emphasize that $x_{c}$
is different for different models e.g. $x_{c}$ can be the ``penalty
parameter'' when $c$ is a ``linear support vector machine'' while
$x_{c}$ can be the ``initial learning rate'' and the ``regularization
parameter'' when $c$ is a ``logistic regression''. Under such
complex search space, single GP-based BO methods such as One-hot-Encoding
and Merchan-Lobato cannot work. In contrast, as discussed earlier
our method straightforwardly works with this setting thanks to fitting
different GPs for different values $c$. We optimize the black-box
function on 30 benchmark datasets\footnote{Download from \url{https://www.openml.org}. Each dataset is randomly
split into 80\% for training and 20\% for testing.}, compared with three well-known state-of-the-art \textit{automated
machine learning} packages, namely \textbf{Hyperopt-sklearn}\footnote{\url{https://github.com/hyperopt/hyperopt-sklearn}}
(using TPE for optimization) \cite{Komer2019}, \textbf{Auto-sklearn}\footnote{\url{https://github.com/automl/auto-sklearn}}
(using SMAC for optimization) \cite{Feurer2019}, and Tree-Based Pipeline
Optimization Tool (\textbf{TPOT}\footnote{\url{https://github.com/EpistasisLab/tpot}})
\cite{Olson2019}. All methods are applied to the same training and
test sets and repeated 10 times.

Table \ref{tab:Individual-accuracy} shows the classification results
on 16 datasets (the results for 30 datasets are reported in the supplementary
material), where \textbf{Bandit-BO} clearly results in better classification
compared with other methods. More specifically, \textbf{Bandit-BO}
achieves up to 4\%, 9\%, and 6\% improvements over Hyperopt-sklearn,
Auto-sklearn, and TPOT respectively. On four large datasets (\textit{electricity},
\textit{wall\_robot\_navigation}, \textit{phoneme}, and \textit{waveform}),
our method is better than three baselines. The improvements are more
significant on five small datasets (\textit{wine}, \textit{breast\_cancer},
\textit{diabetes}, \textit{vehicle}, and \textit{balance\_scale}).
The classification performances of Hyperopt-sklearn, Auto-sklearn,
and TPOT are comparable. In the supplementary material, we report
the overall accuracy of each method across 30 datasets, where \textbf{Bandit-BO}
is the best method (92.25\% accuracy). The overall classification
results of Hyperopt-sklearn (91.48\% accuracy), Auto-sklearn (91.82\%
accuracy), and TPOT (91.77\% accuracy) are quite similar.

%% file: conclusion.tex
We have introduced a novel BO method to globally optimize expensive
black-box functions involving both categorical and continuous variables.
We formulated the problem as a MAB problem, where each category corresponds
to an arm with its reward distribution centered around the optimum
of the objective function in continuous variables. Our solution uses
Thompson sampling, which connects both MAB and BO in a unified framework.
Our method is capable of handling optimization problems where each
category is associated with a different continuous search space. We
also extended our method for batch optimization. We rigorously analyzed
the convergence providing sub-linear regret bounds. Our experiments
using several synthetic and real-world applications demonstrate the
usefulness of our proposed method.

%% file: aaai2020.bbl
\begin{thebibliography}{}

\bibitem[\protect\citeauthoryear{Agrawal and Goyal}{2013}]{agrawal2013further}
Agrawal, S., and Goyal, N.
\newblock 2013.
\newblock Further optimal regret bounds for thompson sampling.
\newblock In {\em AISTATS},  99--107.

\bibitem[\protect\citeauthoryear{Auer, Cesa-Bianchi, and
  Fischer}{2002}]{auer2002finite}
Auer, P.; Cesa-Bianchi, N.; and Fischer, P.
\newblock 2002.
\newblock Finite-time analysis of the multiarmed bandit problem.
\newblock {\em Machine Learning} 47(2-3):235--256.

\bibitem[\protect\citeauthoryear{Bergstra \bgroup et al\mbox.\egroup
  }{2011}]{bergstra2011algorithms}
Bergstra, J.; Bardenet, R.; Bengio, Y.; and K{\'e}gl, B.
\newblock 2011.
\newblock Algorithms for hyper-parameter optimization.
\newblock In {\em NIPS},  2546--2554.

\bibitem[\protect\citeauthoryear{Bonilla, Chai, and
  Williams}{2008}]{Bonilla_etal_08multi}
Bonilla, E.~V.; Chai, K.~M.; and Williams, C.
\newblock 2008.
\newblock Multi-task gaussian process prediction.
\newblock In {\em Advances in neural information processing systems},
  153--160.

\bibitem[\protect\citeauthoryear{Bubeck and Liu}{2013}]{Bubeck_Liu_2013prior}
Bubeck, S., and Liu, C.-Y.
\newblock 2013.
\newblock Prior-free and prior-dependent regret bounds for thompson sampling.
\newblock In {\em NIPS},  638--646.

\bibitem[\protect\citeauthoryear{Bubeck, Cesa-Bianchi, and
  others}{2012}]{bubeck2012regret}
Bubeck, S.; Cesa-Bianchi, N.; et~al.
\newblock 2012.
\newblock Regret analysis of stochastic and nonstochastic multi-armed bandit
  problems.
\newblock {\em Foundations and Trends{\textregistered} in Machine Learning}
  5(1):1--122.

\bibitem[\protect\citeauthoryear{Bull}{2011}]{bull2011convergence}
Bull, A.
\newblock 2011.
\newblock Convergence rates of efficient global optimization algorithms.
\newblock {\em Journal of Machine Learning Research} 12:2879--2904.

\bibitem[\protect\citeauthoryear{Desautels, Krause, and
  Burdick}{2014}]{desautels2014parallelizing}
Desautels, T.; Krause, A.; and Burdick, J.
\newblock 2014.
\newblock Parallelizing exploration-exploitation tradeoffs in gaussian process
  bandit optimization.
\newblock {\em The Journal of Machine Learning Research} 15(1):3873--3923.

\bibitem[\protect\citeauthoryear{Feurer \bgroup et al\mbox.\egroup
  }{2019}]{Feurer2019}
Feurer, M.; Klein, A.; Eggensperger, K.; Springenberg, J.; Blum, M.; and
  Hutter, F.
\newblock 2019.
\newblock {\em Auto-sklearn: Efficient and Robust Automated Machine Learning}.
\newblock Springer International Publishing.
\newblock  113--134.

\bibitem[\protect\citeauthoryear{Garrido-Merch{\'a}n and
  Hern{\'a}ndez-Lobato}{2018}]{garrido2018dealing}
Garrido-Merch{\'a}n, E., and Hern{\'a}ndez-Lobato, D.
\newblock 2018.
\newblock Dealing with categorical and integer-valued variables in bayesian
  optimization with gaussian processes.
\newblock {\em arXiv preprint arXiv:1805.03463}.

\bibitem[\protect\citeauthoryear{Golovin \bgroup et al\mbox.\egroup
  }{2017}]{golovin2017google}
Golovin, D.; Solnik, B.; Moitra, S.; Kochanski, G.; Karro, J.; and Sculley, D.
\newblock 2017.
\newblock Google vizier: A service for black-box optimization.
\newblock In {\em KDD},  1487--1495.
\newblock ACM.

\bibitem[\protect\citeauthoryear{Gonz{\'a}lez \bgroup et al\mbox.\egroup
  }{2016}]{gonzalez2016batch}
Gonz{\'a}lez, J.; Dai, Z.; Hennig, P.; and Lawrence, N.
\newblock 2016.
\newblock Batch bayesian optimization via local penalization.
\newblock In {\em AISTATS},  648--657.

\bibitem[\protect\citeauthoryear{Hern{\'a}ndez-Lobato \bgroup et
  al\mbox.\egroup }{2017}]{hernandez2017parallel}
Hern{\'a}ndez-Lobato, M.; Requeima, J.; Pyzer-Knapp, E.; and Aspuru-Guzik, A.
\newblock 2017.
\newblock Parallel and distributed thompson sampling for large-scale
  accelerated exploration of chemical space.
\newblock In {\em ICML},  1470--1479.

\bibitem[\protect\citeauthoryear{Hern{\'a}ndez-Lobato, Hoffman, and
  Ghahramani}{2014}]{hernandez2014predictive}
Hern{\'a}ndez-Lobato, M.; Hoffman, M.; and Ghahramani, Z.
\newblock 2014.
\newblock Predictive entropy search for efficient global optimization of
  black-box functions.
\newblock In {\em NIPS},  918--926.

\bibitem[\protect\citeauthoryear{Hutter, Hoos, and
  Leyton-Brown}{2011}]{hutter2011sequential}
Hutter, F.; Hoos, H.; and Leyton-Brown, K.
\newblock 2011.
\newblock Sequential model-based optimization for general algorithm
  configuration.
\newblock In {\em International Conference on Learning and Intelligent
  Optimization},  507--523.
\newblock Springer.

\bibitem[\protect\citeauthoryear{Jones, Schonlau, and
  Welch}{1998}]{jones1998efficient}
Jones, D.; Schonlau, M.; and Welch, W.
\newblock 1998.
\newblock Efficient global optimization of expensive black-box functions.
\newblock {\em Journal of Global optimization} 13(4):455--492.

\bibitem[\protect\citeauthoryear{Kandasamy \bgroup et al\mbox.\egroup
  }{2018}]{kandasamy2018parallelised}
Kandasamy, K.; Krishnamurthy, A.; Schneider, J.; and P{\'o}czos, B.
\newblock 2018.
\newblock Parallelised bayesian optimisation via thompson sampling.
\newblock In {\em AISTATS},  133--142.

\bibitem[\protect\citeauthoryear{Kingma and Ba}{2015}]{kingma2014adam}
Kingma, D., and Ba, J.
\newblock 2015.
\newblock {Adam: A method for stochastic optimization}.
\newblock In {\em ICLR}.

\bibitem[\protect\citeauthoryear{Komer, Bergstra, and
  Eliasmith}{2019}]{Komer2019}
Komer, B.; Bergstra, J.; and Eliasmith, C.
\newblock 2019.
\newblock {\em Hyperopt-Sklearn}.
\newblock Springer International Publishing.
\newblock  7--111.

\bibitem[\protect\citeauthoryear{Kushner}{1964}]{kushner1964new}
Kushner, H.
\newblock 1964.
\newblock A new method of locating the maximum point of an arbitrary multipeak
  curve in the presence of noise.
\newblock {\em Journal of Basic Engineering} 86(1):97--106.

\bibitem[\protect\citeauthoryear{Lakshminarayanan, Roy, and
  Teh}{2016}]{Lakshminarayanan_etal_16mondrian}
Lakshminarayanan, B.; Roy, D.~M.; and Teh, Y.~W.
\newblock 2016.
\newblock Mondrian forests for large-scale regression when uncertainty matters.
\newblock In {\em Artificial Intelligence and Statistics},  1478--1487.

\bibitem[\protect\citeauthoryear{Olson and Moore}{2019}]{Olson2019}
Olson, R., and Moore, J.
\newblock 2019.
\newblock {\em TPOT: A Tree-Based Pipeline Optimization Tool for Automating
  Machine Learning}.
\newblock Springer International Publishing.
\newblock  151--160.

\bibitem[\protect\citeauthoryear{Russo and Roy}{2014}]{russo2014learning}
Russo, D., and Roy, B.
\newblock 2014.
\newblock Learning to optimize via posterior sampling.
\newblock {\em Mathematics of Operations Research} 39(4):1221--1243.

\bibitem[\protect\citeauthoryear{Shahriari \bgroup et al\mbox.\egroup
  }{2016}]{shahriari2016taking}
Shahriari, B.; Swersky, K.; Wang, Z.; Adams, R.; and Freitas, N.
\newblock 2016.
\newblock Taking the human out of the loop: A review of bayesian optimization.
\newblock {\em Proceedings of the IEEE} 104(1):148--175.

\bibitem[\protect\citeauthoryear{Srinivas \bgroup et al\mbox.\egroup
  }{2012}]{srinivas2012information}
Srinivas, N.; Krause, A.; Kakade, S.; and Seeger, M.
\newblock 2012.
\newblock Information-theoretic regret bounds for gaussian process optimization
  in the bandit setting.
\newblock {\em IEEE Transactions on Information Theory} 58(5):3250--3265.

\bibitem[\protect\citeauthoryear{Wang \bgroup et al\mbox.\egroup
  }{2016}]{wang2016bayesian}
Wang, Z.; Hutter, F.; Zoghi, M.; Matheson, D.; and de~Feitas, N.
\newblock 2016.
\newblock Bayesian optimization in a billion dimensions via random embeddings.
\newblock {\em Journal of Artificial Intelligence Research} 55:361--387.

\end{thebibliography}
